\documentclass[11pt,oneside]{article}

\usepackage[numbers]{natbib}

\usepackage[utf8]{inputenc}
\usepackage[T1]{fontenc}
\usepackage{nicefrac}
\usepackage{amsmath}
\usepackage{amssymb,amsfonts}
\usepackage{bm}

\usepackage[paper=a4paper, hscale=0.75, vscale=0.79]{geometry}

\usepackage[dvipsnames]{xcolor}

\usepackage{graphicx}
\graphicspath{{./}}
\DeclareGraphicsExtensions{.pdf,.png,.jpg,.eps}
\usepackage{subfig}

\DeclareMathOperator{\exx}{\mathbf{E}}

\DeclareMathOperator{\prr}{\mathbf{P}}

\DeclareMathOperator*{\argmin}{arg\,min}

\def\FF{\mathcal{F}}

\def\HH{\mathcal{H}}

\def\RR{\mathbb{R}}

\def\VV{\mathcal{V}}

\newcommand{\overbar}[1]{\mkern 1.5mu\overline{\mkern-1.5mu#1\mkern-1.5mu}\mkern 1.5mu}

\newcommand*{\defeq}{\mathrel{\vcenter{\baselineskip0.5ex \lineskiplimit0pt     
                     \hbox{\scriptsize.}\hbox{\scriptsize.}}}
                     =}

\newcommand{\term}[1]{\textcolor{BlueViolet}{\textit{{#1}}}}

\DeclareMathOperator{\bregbase}{D} % base symbol for bregman divergence.
\def\breg{\bregbase_{\Phi}} % Bregman divergence induced by Phi.
\def\cond{\,|\,} % conditional probability bar with nice spacing.
\def\algo{\mathcal{A}} % generic algorithm notation.
\def\diameter{\Delta} % diameter.
\DeclareMathOperator{\dist}{dist} % distance function.
\DeclareMathOperator{\dom}{dom} % effective domain.
\def\hbar{\overbar{h}} % averages iterates.
\def\hstar{h^{\ast}} % reference point.
\DeclareMathOperator{\idc}{I} % indicator function.
\DeclareMathOperator{\loss}{L} % loss function.
\DeclareMathOperator{\proj}{\Pi} % projection operator.
\DeclareMathOperator{\regret}{Regret} % (linear) regret bound.
\DeclareMathOperator{\risk}{R} % risk.
\def\smooth{\lambda} % smoothness parameter.
\def\strong{\kappa} % strong convexity parameter.
\def\thres{c} % threshold parameter.

\usepackage{amsthm}
\theoremstyle{definition} \newtheorem{defn}{Definition}
\theoremstyle{plain} 
\theoremstyle{plain} \newtheorem{thm}[defn]{Theorem}
\theoremstyle{plain} \newtheorem{lem}[defn]{Lemma}
\theoremstyle{plain} \newtheorem{cor}[defn]{Corollary}
\theoremstyle{remark} 
\theoremstyle{remark}

\usepackage[pdfauthor={Matthew J. Holland},%
colorlinks=true,%
linkcolor=blue,%
citecolor=blue]{hyperref}
\hypersetup{pdftitle={Robust learning with anytime-guaranteed feedback}}

\usepackage{url}
\usepackage{booktabs}
\usepackage{amsfonts}
\usepackage{microtype}
\usepackage{algorithm}
\usepackage{algpseudocode}
\usepackage{enumitem}
\makeatletter
\def\namedlabel#1#2{\begingroup
    #2%
    \def\@currentlabel{#2}%
    \phantomsection\label{#1}\endgroup
}
\makeatother

\title{\textbf{Robust learning with anytime-guaranteed feedback}}
\author{
  Matthew J.~Holland\thanks{Please direct correspondence to \texttt{matthew-h@ar.sanken.osaka-u.ac.jp}.}\\
  Osaka University
}
\date{} % empty date.

\begin{document}

\maketitle

\begin{abstract}
Under data distributions which may be heavy-tailed, many stochastic gradient-based learning algorithms are driven by feedback queried at points with almost no performance guarantees on their own. Here we explore a modified ``anytime online-to-batch'' mechanism which for smooth objectives admits high-probability error bounds while requiring only lower-order moment bounds on the stochastic gradients. Using this conversion, we can derive a wide variety of ``anytime robust'' procedures, for which the task of performance analysis can be effectively reduced to regret control, meaning that existing regret bounds (for the bounded gradient case) can be robustified and leveraged in a straightforward manner. As a direct takeaway, we obtain an easily implemented stochastic gradient-based algorithm for which all queried points formally enjoy sub-Gaussian error bounds, and in practice show noteworthy gains on real-world data applications.
\end{abstract}

\tableofcontents

\section{Introduction}\label{sec:intro}

The ultimate goal of many learning tasks can be formulated as a minimization problem:
\begin{align}
\min_{h} \, \risk(h), \text{ s.t. } h \in \HH.
\end{align}
What characterizes this as a learning problem is that $\risk$ (henceforth called the \term{true objective}) is \emph{unknown} to the learner, who must choose from the hypothesis class $\HH$ a final candidate based only on incomplete and noisy (stochastic) feedback related to $\risk$ \citep{haussler1992a,vapnik1999NSLT}. One of the most ubiquitous and well-studied feedback mechanisms is the \term{stochastic gradient oracle} \citep{hazan2016OCO,nemirovsky1983a,shalev2012a}, in which the learner generates a sequence of candidates $(h_t)$ based on a sequence of random sub-gradients $(G_t)$, which are unbiased in the following sense:
\begin{align}\label{eqn:feedback_naive}
\exx\left[ G_t \cond G_{[t-1]} \right] \in \partial\risk(h_t), \text{ for all } t \geq 1.
\end{align}
Here $\partial\risk(h)$ denotes the sub-differential of $\risk$ evaluated at $h$, and we denote sub-sequences by $G_{[t]} \defeq (G_1,\ldots,G_{t})$.\footnote{More strictly speaking, for each $t$, this inclusion holds almost surely over the random draw of $G_{[t]}$, and the conditional expectation is that of $G_t$ conditioned on the sigma-algebra generated by $G_{[t-1]}$. See \citet[Ch.~5--6]{ash2000a} for additional background on probabilistic foundations.} Our problem of interest is that of efficiently minimizing $\risk(\cdot)$ over $\HH$ when the noisy feedback is \emph{potentially heavy-tailed}, i.e., for all steps $t$, it is unknown whether the distribution of $G_t$ is congenial in the sub-Gaussian sense, or heavy-tailed in the sense of having infinite or undefined higher-order moments \citep{chen2017a,holland2019c}. By ``efficiently,'' we mean procedures with performance guarantees (high-probability error bounds) on par with the case in which the learner knows \textit{a priori} that the feedback is sub-Gaussian \citep{devroye2016a,nazin2019a}.

Recently, notable progress has been made on this front, with a common theme of making principled modifications (e.g., truncation, data splitting + validation, etc.) to the the raw feedback $(G_t)$ before passing it to a more traditional stochastic gradient-based update, to achieve sub-Gaussian bounds while assuming just finite variance \citep{davis2019a,gorbunov2020a,nazin2019a}. Here we focus on two key limitations to the current state of the art: (a) many robust learning algorithms only have such guarantees when $\risk$ is \emph{strongly convex} \citep{chen2017a,davis2019a,holland2019c}; (b) without strong convexity, sub-Gaussian guarantees are unavailable for the iterates $(h_t)$ being queried in (\ref{eqn:feedback_naive}), only for a running average of these iterates \citep{gorbunov2020a,nazin2019a}. While there exist general-purpose ``anytime'' online-to-batch conversions to ensure that the points being queried have guarantees \citep{cutkosky2019a}, even the most refined conversions either require bounded gradients or are only in expectation \citep{joulani2020a}, meaning that under potentially heavy-tailed gradients, a direct anytime conversion based on existing results fails to achieve the desired guarantees.

In this paper, in order to address the issues described above, we introduce a modified mechanism for making the ``anytime'' conversion (Algorithm \ref{algo:anytime_robust}), which is both easy to implement and robust to the underlying data distribution. More concretely, assuming only that $\risk$ is convex and smooth, and that raw gradients have finite variance, we obtain martingale concentration guarantees for truncated gradients queried at a moving average (Lemma \ref{lem:grad_error_highprob}), which lets us reduce the problem of obtaining error bounds to that of regret control (section \ref{sec:anytime_robust}), substantially broadening the domain to which the anytime conversion of \citet{cutkosky2019a} can be applied. Regret control for online learning algorithms (under \emph{bounded} gradients) is a well-studied problem, and in section \ref{sec:anytime_algos} we show that existing well-known regret bounds can be readily modified to utilize the control offered by Lemma \ref{lem:grad_error_highprob}. In particular, we look at vanilla FTRL (Lemma \ref{lem:rftrl_regret}), mirror descent (Lemma \ref{lem:rsmd_regret}), and AO-FTRL (Theorem \ref{thm:raoftrl_excess_risk}), giving us ``anytime robust'' analogues to results given in expectation by \citet{joulani2020a}. As a natural takeaway, we obtain a stochastic gradient-based procedure (section \ref{sec:rsmd}) for which \emph{all queried points} have sub-Gaussian error bounds (Corollary \ref{cor:anytime_sgd}), a methodological improvement over the averaging scheme of \citet{nazin2019a}, which we also empirically demonstrate has substantial practical benefits (section \ref{sec:empirical}). Our results are stated with a high degree of generality (works on any reflexive Banach space), and taken together are suggestive of an appealing general-purpose learning strategy.

\section{Preliminaries}\label{sec:prelims}

\paragraph{The underlying space}

For the underlying hypothesis class $\HH$, we shall assume $\HH \subset \VV$, where $(\VV,\|\cdot\|)$ is a normed linear space. For any normed space $(\VV,\|\cdot\|)$, we will denote by $\VV^{\ast}$ the usual dual of $\VV$, namely the set of all continuous linear functionals on $\VV$. As is traditional, the norm for the dual is defined $\|f\|_{\ast} \defeq \sup\{a : |f(h)| \leq a\|v\|, v \in \VV \}$ for $f \in \VV^{\ast}$. We denote the distance from a point $v$ to a set $A \subset \VV$ by $\dist(v;A) \defeq \inf\{\|v-u\|: u \in A\}$. We use the notation $\langle \cdot, \cdot \rangle$ to represent the coupling function between $\VV$ and $\VV^{\ast}$, namely $\langle v^{\ast}, v \rangle \defeq v^{\ast}(v)$ for each $v^{\ast} \in \VV^{\ast}$ and all $v \in \VV$; when $\VV$ is a Hilbert space this coincides with the usual inner product. We denote the extended real line by $\overbar{\RR}$.

\paragraph{Convexity and smoothness}

We say that a function $f:\VV \to \overbar{\RR}$ is \term{convex} if for all $0 < \alpha < 1$ and $u,v \in \VV$, we have $f(\alpha u + (1-\alpha)v) \leq \alpha f(u) + (1-\alpha)f(v)$. The \term{effective domain} of $f$ is defined $\dom f \defeq \{u \in \VV: f(u) < \infty\}$. A convex function $f: \VV \to \overbar{\RR}$ is said to be \term{proper} if $-\infty < f$ and $\dom f \neq \emptyset$. For any proper convex function $f: \VV \to \overbar{\RR}$, the \term{sub-differential} of $f$ at $h \in \VV$ is $\partial f(h) \defeq \left\{ v^{\ast} \in \VV^{\ast}: f(u)-f(h) \geq \langle v^{\ast}, u-h \rangle, u \in \VV \right\}$. For readability, we will sometimes make statements involving multi-valued functions; for example, the statement ``$\langle \partial f(h), u \rangle = a$,'' is equivalent to the statement ``$\langle v^{\ast}, u \rangle = a$ for all $v^{\ast} \in \partial f(h)$.'' When we say a certain point $\hstar$ is a \term{stationary point} of $f$ on $\HH$, we mean that $\hstar \in \HH$ and $0 \in \partial f(\hstar)$. If the convex function $f$ happens to be (Gateaux) differentiable at some $h \in \HH$, then the sub-differential contains a unique element, $\partial f(h) = \{\nabla f(h)\}$, the \term{gradient} of $f$ at $h$. When we say that $f$ is \term{$\smooth$-smooth} on some open convex set $U \subset \VV$, we mean that $\|\nabla f(h)-\nabla f(h^{\prime})\|_{\ast} \leq \smooth\|h-h^{\prime}\|$ for all $h,h^{\prime} \in U$. For any sub-differentiable function $f$, we write $\bregbase_{f}(u;v) \defeq f(u)-f(v)-\langle \partial f(v), u-v \rangle$; when $f$ happens to be convex and differentiable, this becomes the usual \term{Bregman divergence} induced by $f$.

\paragraph{Miscellaneous notation}

For indexing purposes, we denote the set of all positive integers no greater than $k$ by $[k] \defeq \{1,\ldots,k\}$. We denote $\alpha_{1:t} \defeq \sum_{i=1}^{t}\alpha_i$ for any integer $t \geq 1$, using the convention $\alpha_{1:0} \defeq 0$ as needed. We also denote sub-sequences in a similar fashion, with $\alpha_{[t]} \defeq (\alpha_1,\ldots,\alpha_t)$; this applies not only to $(\alpha_t)$, but also $(h_t)$, $(G_t)$ and other sequences used throughout the paper. Indicator functions (i.e., Bernoulli random variables) are typically denoted as $\idc\{\texttt{event}\}$.

\paragraph{Anytime conversions}

As preparation, we start with almost no assumptions on the learning algorithm or feedback-generating process. Let $(h_t)$ be an arbitrary sequence of candidates, henceforth referred to as the \term{ancillary iterates}. Letting $(\alpha_t)$ be a sequence of positive weights, we consider the corresponding \term{main iterates} $(\hbar_t)$, defined for all $t \geq 1$ as
\begin{align}\label{eqn:hbar_defn}
\hbar_t = \texttt{Weighting}\left[ h_t; h_{[t-1]} \right] \defeq \frac{\sum_{i=1}^{t} \alpha_i h_i}{\alpha_{1:t}}.
\end{align}
As a starting point, we note that the excess error of the weighted main iterates can be expressed in a convenient fashion.
\begin{lem}[Anytime lemma]\label{lem:anytime_basic}
Let $\VV$ be a linear space, and let $\risk:\VV \to \overbar{\RR}$ be sub-differentiable. Let $(h_t)$ be an arbitrary sequence of $h_t \in \dom \risk$, and let $(\hbar_t)$ be generated via (\ref{eqn:hbar_defn}). Then we have
\begin{align*}
\risk(\hbar_T) - \risk(\hstar) = \frac{1}{\alpha_{1:T}} \left[ \sum_{t=1}^{T} \alpha_t \left[ \langle \partial\risk(\hbar_t), h_t - \hstar \rangle - \bregbase_{\risk}(\hstar;\hbar_t) \right] - \sum_{t=1}^{T-1}\alpha_{1:t}\bregbase_{\risk}(\hbar_t;\hbar_{t+1}) \right]
\end{align*}
for any reference point $\hstar \in \dom \risk$ and $T \geq 1$.
\end{lem}
\noindent The above equality is a slight generalization of the anytime online-to-batch inequality introduced by \citet{cutkosky2019a} and sharpened by \citet{joulani2020a}; it follows by direct manipulations utilizing little more than the definition of $\bregbase_{\risk}$. The key point of Lemma \ref{lem:anytime_basic} is that we can obtain control over the main iterates $\risk(\hbar_t)$ using an ideal quantity that depends directly on $\hbar_t$, rather than simply $h_t$, as is typical of traditional online-to-batch conversions \citep{cesa2004a}. This is important because it opens the door to new stochastic feedback processes, driven by the \emph{main} iterates, rather than the ancillary ones. In other words, we want feedback that provides an estimate of some element of $\partial\risk(\hbar_t)$, rather than $\partial\risk(h_t)$. When $\risk$ is convex, we have $\bregbase_{\risk} \geq 0$, and the subtracted terms can be utilized to sharpen our guarantees once we have regret bounds, as will be discussed in the technical appendices.

\section{Anytime robust algorithm design}\label{sec:anytime_robust}

In Algorithm \ref{algo:anytime_robust}, we give a summary of the modified online-to-batch conversion that we utilize throughout the rest of the paper. Essentially, we start with an arbitrary online learning algorithm $\algo$, query the potentially heavy-tailed stochastic feedback after averaging the iterates, and process the raw gradients in a robust fashion before updating. In the following paragraphs, we describe the details of these steps.
\begin{algorithm}[t!]
\caption{Anytime robust online-to-batch conversion.}
\label{algo:anytime_robust}
\begin{algorithmic}
\State \textbf{inputs:} Weights $(\alpha_t)$, thresholds $(\thres_t)$, algorithm $\algo$, initial point $h_1$, max iterations $T$.
\State Initialize $\hbar_1 = h_1$.
\For{$t \in [T-1]$}
 \State Obtain stochastic gradient $G_t$ at $\hbar_t$, satisfying (\ref{eqn:grad_condition}). %\Comment{Critical for anytime guarantee.}
 \State Set $\displaystyle \overbar{G}_t = \texttt{Process}[G_t; \thres_t]$ following (\ref{eqn:truncate_defn})--(\ref{eqn:truncate_anchors}).% \Comment{Critical for robustness.}
 \State Ancillary update: $h_{t+1} = \algo(h_t)$.
% \State Compute loss: $\loss_t = \alpha_t \langle \overbar{G}_t, h_t \rangle$.
% \State Ancillary update: $h_{t+1} = \algo(\loss_t)$.
 \State Main update: $\hbar_{t+1} = \texttt{Weighting}[h_{t+1}; h_{[t]}]$, as in (\ref{eqn:hbar_defn}).%\sum_{i=1}^{t+1} \alpha_i h_i / \alpha_{1:(t+1)}$.
\EndFor
\State \textbf{return:} $\displaystyle \hbar_T$.
\end{algorithmic}
\end{algorithm}

\paragraph{Raw feedback process}

Let $(G_t)$ denote a sequence of stochastic gradients $G_t \in \VV^{\ast}$, which are conditionally unbiased in the sense that we have
\begin{align}\label{eqn:grad_condition}
\exx_{[t-1]}G_t \defeq \exx \left[ G_t \cond G_{[t-1]} \right] \in \partial \risk(\hbar_t)
\end{align}
for all $t \geq 1$, recalling our notation $G_{[t]} \defeq (G_1,\ldots,G_t)$. We emphasize to the reader that (\ref{eqn:grad_condition}) differs from the traditional assumption (\ref{eqn:feedback_naive}) in terms of the points at which the sub-differential is being evaluated ($\hbar_t$ rather than $h_t$). As is traditional in the literature \citep{nazin2019a,nguyen2018a}, we shall also assume a uniform bound on the conditional variance, namely that for all $t \geq 1$, we have
\begin{align}\label{eqn:grad_variance_bound}
\exx_{[t-1]}\|G_t - \exx_{[t-1]}G_t\|_{\ast} \leq \sigma^{2} < \infty.
\end{align}
We will not assume anything else about the underlying distribution of $(G_t)$; as such, the gradients clearly may be unbounded or heavy-tailed in the sense of having infinite or undefined higher-order moments. In this setting, while one could naively use the raw sequence $(G_t)$ as-is, since we have made extremely weak assumptions on the underlying distribution, it is always possible for heavy-tailed data to severely destabilize the learning process \citep{brownlees2015a,chen2017a,lecue2018a}. As such, it is desirable to process the raw gradients in a statistically principled manner, such that the processed output provides useful feedback to be passed directly to $\algo$.

\paragraph{Robust feedback design}

A simple and popular approach to deal with heavy-tailed random vectors is to use norm-based truncation \citep{catoni2017a,nazin2019a}. As with \citet{nazin2019a}, we process the raw gradients as follows:
\begin{align}\label{eqn:truncate_defn}
\texttt{Process}\left[ G; \thres \right] \defeq 
\begin{cases}
\widetilde{g}, & \text{if } \|G - \widetilde{g}\|_{\ast} > \thres\\
G, & \text{else}.
\end{cases}
\end{align}
Here $c>0$ is a threshold, the point $\widetilde{g} \in \VV^{\ast}$ used in this sub-routine is an ``anchor'' in the dual space, associated with some ``primal anchor'' $\widetilde{h} \in \HH$ assumed to satisfy
\begin{align}\label{eqn:truncate_anchors}
\prr\left\{ \dist(\widetilde{g};\partial\risk(\widetilde{h})) > \widetilde{\varepsilon}\sigma \right\} \leq \delta.
\end{align}
We discuss settings of $\widetilde{h}$ and $\delta \in (0,1)$ in section \ref{sec:empirical}. To summarize, instead of naively using $(G_t)$ as feedback for $\algo$, we will pass $(\overbar{G}_t)$, defined by $\overbar{G}_t \defeq \texttt{Process}\left[ G_t; \thres_t \right]$, based on a sequence of thresholds $(\thres_t)$. The anchors $\widetilde{g}$ and $\widetilde{h}$ remain fixed throughout the learning process.

\paragraph{Estimation error under smooth objectives}

Let us further assume that $\risk$ is $\smooth$-smooth, still leaving $\algo$ abstract. In this case, the sub-differential is simply $\partial\risk(h) = \{\nabla \risk(h)\}$, and so the error that we focus on is naturally that of the approximation $\overbar{G}_t \approx \nabla\risk(\hbar_t)$, for $t \in [T]$. With $(\overbar{G}_t)$ generated as described in Algorithm \ref{algo:anytime_robust}, direct inspection shows us that
\begin{align}
\label{eqn:grad_est_handy_form}
\overbar{G}_t - \nabla\risk(\hbar_t) & = (G_t - \widetilde{g})(1-\idc_t) + \widetilde{g} - \nabla\risk(\hbar_t)
\end{align}
where $\idc_t$ is the Bernoulli random variable defined $\idc_t \defeq \idc\left\{ \|G_t - \widetilde{g}\|_{\ast} > \thres_t \right\}$. The right-hand side of (\ref{eqn:grad_est_handy_form}) has two terms we need to control. The first term is clearly bounded above by $\thres_t$, considering the truncation event. As for the second term, a smooth risk makes it easy to establish control in primal distance terms. More explicitly, we have
\begin{align}
\label{eqn:grad_anchor_bound}
\|\widetilde{g} - \nabla\risk(\hbar_t)\|_{\ast} \leq \|\widetilde{g} - \nabla\risk(\widetilde{h})\|_{\ast} + \|\nabla\risk(\widetilde{h}) - \nabla\risk(\hbar_t)\|_{\ast} \leq \widetilde{\varepsilon}\sigma + \smooth\|\widetilde{h}-\hbar_t\|,
\end{align}
where the latter inequality follows from $\smooth$-smoothness and the anchor property (\ref{eqn:truncate_anchors}). Taking (\ref{eqn:grad_est_handy_form}) and (\ref{eqn:grad_anchor_bound}) together, we readily obtain
\begin{align}
\label{eqn:grad_error_bd}
\|\overbar{G}_t - \nabla\risk(\hbar_t)\|_{\ast} \leq \|G_t - \widetilde{g}\|_{\ast}(1-\idc_t) + \|\widetilde{g} - \nabla\risk(\hbar_t)\|_{\ast} \leq \thres_t + \widetilde{\varepsilon}\sigma + \smooth\|\widetilde{h}-\hbar_t\|
\end{align}
on an event of probability at least $1-\delta$. This inequality suggests an obvious choice for the threshold $\thres_t$ that keeps the preceding upper bound tidy:
\begin{align}\label{eqn:truncate_threshold_smooth}
\thres_t = \widetilde{\varepsilon}\sigma + \smooth\|\widetilde{h}-\hbar_t\| + \thres_0, \quad t \in [T].
\end{align}
Here $\thres_0 > 0$ is positive parameter that is used to control the degree of bias incurred due to truncation. Using this thresholding strategy, one can obtain sub-linear bounds on the weighted gradient error terms, as the next result shows.
\begin{lem}\label{lem:grad_error_highprob}
Let $\risk$ be convex and $\smooth$-smooth. Let $\HH \subset \dom \risk$ be convex with diameter $\diameter < \infty$. Given confidence parameter $0 < \delta < 1$ and iterations $T \geq \log(\delta^{-1})(\lceil \widetilde{\varepsilon}\sigma \rceil)^{2}$, running Algorithm \ref{algo:anytime_robust} with thresholds $(\thres_t)$ as in (\ref{eqn:truncate_threshold_smooth}) with $\thres_0 = \max\{\smooth\diameter, \sigma\sqrt{T/\log(\delta^{-1})}\}+\widetilde{\varepsilon}\sigma$, and weights $(\alpha_t)$ such that $\exx_{[t-1]}\alpha_t = \alpha_t$ almost surely, it follows that
\begin{align*}
\sum_{t=1}^{T}\alpha_t \sup_{h,h^{\prime} \in \HH} \left[ \langle \overbar{G}_t - \nabla\risk(\hbar_t), h-h^{\prime} \rangle \right] \leq \max\left\{ q_{\delta}(T), r_{\delta}(T) \right\}
\end{align*}
with probability no less than $1-2\delta$, where we have defined
\begin{align*}
q_{\delta}(T) & \defeq 2\diameter\sigma\sqrt{2\log(\delta^{-1})}\left[ \frac{\alpha_{1:T}}{\sqrt{T}} + \sqrt{\sum_{t=1}\alpha_t^{2}} + 2\left(\max_{t \in [T]}\alpha_t\right) \right]\\
r_{\delta}(T) & \defeq 2\smooth\diameter^{2}\log(\delta^{-1})\left[ \frac{\alpha_{1:T}}{T} + \sqrt{\frac{1}{T}\sum_{t=1}\alpha_t^{2}} + 2\sqrt{2}\left(\max_{t \in [T]}\alpha_t\right) \right].
\end{align*}
\end{lem}
\noindent The main benefit of this lemma is that it holds under very weak assumptions on the stochastic gradients. The main limitations are that the feasible set has a finite diameter, and prior knowledge of $T$ and other factors are used for thresholding.

\paragraph{A general strategy}

Let us define the regret incurred by Algorithm \ref{algo:anytime_robust} after $T$ steps by
\begin{align}\label{eqn:regret_defn}
\regret(T;\algo) \defeq \sum_{t=1}^{T}\alpha_t \langle \overbar{G}_t, h_t-\hstar \rangle,
\end{align}
where the reference point $\hstar \in \dom \risk$ is left implicit in the notation. This weighted linear regret is somewhat special since the losses (i.e., $h \mapsto \alpha_t\langle \overbar{G}_t, h \rangle $) are \emph{evaluated} on the ancillary sequence $(h_t)$, but they are \emph{defined} in terms of potentially biased stochastic gradients which depend on the main sequence $(\hbar_t)$. With this notion of regret in hand, note that from Lemma \ref{lem:anytime_basic}, we immediately have the following expression:
\begin{align}
\nonumber
\risk(\hbar_T) - \risk(\hstar) = \frac{1}{\alpha_{1:T}} & \left[ \regret(T;\algo) + \sum_{t=1}^{T} \alpha_t \langle \overbar{G}_t - \nabla\risk(\hbar_t), \hstar - h_t\rangle \right.\\
\label{eqn:regret_starting_point}
& \qquad \left. - \sum_{t=1}^{T}\alpha_t \bregbase_{\risk}(\hstar;\hbar_t) - \sum_{t=1}^{T-1}\alpha_{1:t}\bregbase_{\risk}(\hbar_t;\hbar_{t+1}) \right].
\end{align}
This inequality offers us a nice starting point for analyzing a wide class of ``anytime robust algorithms,'' since the second sum can clearly be controlled using Lemma \ref{lem:grad_error_highprob}. It just remains to seek out regret bounds for different choices of the underlying algorithm $\algo$ which are sub-linear, up to error terms that are amenable to Lemma \ref{lem:grad_error_highprob}. We give several concrete examples in the next section. To close this section, by combining our notion of regret with the preceding lemma, we can obtain a ``robust'' analogue of \citet[Thm.~1]{cutkosky2019a}, which is valid under unbounded, heavy-tailed stochastic gradients.
\begin{cor}\label{cor:anytime_highprob}
Under the assumptions of Lemma \ref{lem:grad_error_highprob}, for any reference point $\hstar \in \HH$, we have
\begin{align*}
\risk(\hbar_T) - \risk(\hstar) \leq \frac{1}{\alpha_{1:T}} \left[ \regret(T;\algo) + \max\left\{ q_{\delta}(T), r_{\delta}(T) \right\} - B_{T} \right]
\end{align*}
with probability no less than $1-2\delta$, where $B_T \geq 0$ denotes the sum of all the Bregman divergence terms given in (\ref{eqn:regret_starting_point}).
\end{cor}

\section{Anytime robust learning algorithms}\label{sec:anytime_algos}

Thus far, the underlying algorithm object $\algo$ used in Algorithm \ref{algo:anytime_robust} has been left abstract. In this section, we illustrate how (\ref{eqn:regret_starting_point}) can be utilized for important classes of algorithms, by obtaining regret bounds that are sub-linear up to error terms that can be controlled using Lemma \ref{lem:grad_error_highprob}. Our running assumptions are that $(\VV,\|\cdot\|)$ is a reflexive Banach space, $\HH \subseteq \VV$ is convex and closed, $\risk$ is sub-differentiable, and the sequence $(\overbar{G}_t)$ driven by $(\hbar_t)$ is precisely as in Algorithm \ref{algo:anytime_robust}.

\subsection{Anytime robust FTRL}

Here we consider the setting in which $\algo$ is implemented using a form of follow-the-regularized-leader (FTRL). Letting $(\psi_t)$ be a sequence of regularizer functions $\psi_t: \VV \to \RR$, we are interested in the ancillary sequence $(h_t)$ generated by
\begin{align}\label{eqn:rftrl_update}
h_{t+1} = \algo(h_t) \in \argmin_{h \in \HH} \left[ \psi_{t+1}(h) + \sum_{i=1}^{t} \alpha_i \langle \overbar{G}_i, h \rangle \right].
\end{align}
The initial value is set using an extra regularizer $\psi_1$, with $h_1 \in \argmin_{h \in \HH} \psi_1(h)$. We proceed assuming that the sequence $(h_t)$ exists, but we do not require the minimizer in (\ref{eqn:rftrl_update}) to be unique.
\begin{lem}\label{lem:rftrl_regret}
Let $\algo$ be implemented as in (\ref{eqn:rftrl_update}), assuming that for each step $t \geq 1$, the regularizer $\psi_t$ is $\strong_t$-strongly convex. Then, for any reference point $\hstar \in \HH$, we have
\begin{align}
\nonumber
\regret(T;\algo) & \leq \psi_T(\hstar) - \psi_1(h_1) + \sum_{t=1}^{T}\left[ \psi_t(h_{t+1}) - \psi_{t+1}(h_{t+1}) \right]\\
\label{eqn:rftrl_regret}
& \qquad + \sum_{t=1}^{T}\left[ \frac{\|\partial\risk(\hbar_t)\|_{\ast}^{2}}{2\strong_t} + \alpha_t \langle \partial\risk(\hbar_t) - \overbar{G}_t, h_{t+1} - h_t \rangle \right].
\end{align}
\end{lem}
\noindent This lemma is a natural anytime robust analogue of standard FTRL regret bounds \citep[Lem.~7.8]{orabona2020a}. While the above bound holds as long as $\risk$ is sub-differentiable, in the special case where $\risk$ is smooth, the final sum on the right-hand side of (\ref{eqn:rftrl_regret}) is amenable to direct application of Lemma \ref{lem:grad_error_highprob}, as desired. Combining this with (\ref{eqn:regret_starting_point}), one can immediately derive excess risk bounds for the output of Algorithm \ref{algo:anytime_robust} under this FTRL-type of implementation, for a wide variety of regularization strategies.

\subsection{Anytime robust SMD}\label{sec:rsmd}

Next we consider the closely related setting in which $\algo$ is implemented using a form of stochastic mirror descent (SMD). Assuming $\HH \subset \VV$ is bounded, closed, and convex, let $\Phi: \VV \to \RR$ be a differentiable and strictly convex function. Let $\algo$ generate $(h_t)$ based on the update
\begin{align}\label{eqn:rsmd_update_proximal}
h_{t+1} = \algo(h_t) = \argmin_{h \in \HH} \left[ \langle \overbar{G}_t, h \rangle + \frac{1}{\beta_t} \breg(h;h_t) \right].
\end{align}
The function $\breg$ is the Bregman divergence induced by $\Phi$; see the appendix for more detailed background. The step sizes $(\beta_t)$ are assumed positive, but can be set freely.
\begin{lem}\label{lem:rsmd_regret}
Let $\algo$ be implemented as in (\ref{eqn:rsmd_update_proximal}), with $\Phi$ chosen to be $\strong$-strongly convex on $\HH$. Then for any reference point $\hstar \in \HH$, we have
\begin{align*}
\langle \overbar{G}_t, h_t - \hstar \rangle \leq \frac{\breg(\hstar;h_t) - \breg(\hstar;h_{t+1})}{\beta_t} + \frac{\beta_{t}}{2\strong}\|\partial\risk(\hbar_t)\|_{\ast}^{2} + \langle \partial\risk(\hbar_t) - \overbar{G}_t, h_{t+1} - h_t \rangle
\end{align*}
for all $t \geq 1$.
\end{lem}
\noindent This lemma can be interpreted easily as an anytime robust analogue of traditional regret bounds for SMD (e.g., \citep[Lem.~6.7]{orabona2020a}). It can be combined with (\ref{eqn:regret_starting_point}) and Lemma \ref{lem:grad_error_highprob} to obtain the following guarantee.
\begin{thm}[Anytime robust mirror descent]\label{thm:anytime_rsmd}
Under the setting of Lemmas \ref{lem:grad_error_highprob} and \ref{lem:rsmd_regret}, denote the diameter of $\HH$ with respect to $\breg$ as $\diameter_{\Phi} \defeq \sup_{h,h^{\prime} \in \HH}\breg(h;h^{\prime}) < \infty$. Setting the weight sequences such that $\alpha_t/\alpha_{t-1} \geq \beta_t/\beta_{t-1}$ and $\beta_t \leq \strong / \smooth$, we have that for any $\hstar$ which is a stationary point of $\risk$ on $\HH$, the inequality
\begin{align*}
\risk(\hbar_T) - \risk(\hstar) \leq \frac{1}{\alpha_{1:T}} \left[ \frac{\alpha_T}{\beta_T}\diameter_{\Phi} + \max\{q_{\delta}(T),r_{\delta}(T)\} \right]
\end{align*}
holds with probability no less than $1-2\delta$.
\end{thm}
\noindent In contrast with \citet{nazin2019a} who query at the ancillary iterates, the preceding high-probability error bounds effectively give us sub-Gaussian guarantees for all points used to query stochastic gradients. As an important special case, consider the setting where $\VV$ is Euclidean space, and the underlying norm used is the $\ell_2$ norm $\|\cdot\|_{2}$. In this case, it is easy to verify that setting $\Phi(u) = \|u\|_{2}^{2}/2$, with $\strong = 1$ the update (\ref{eqn:rsmd_update_proximal}) amounts to
\begin{align}\label{eqn:sgd_update}
h_{t+1} = \algo(h_t) = \proj_{\HH}\left[ h_t - \beta_t \overbar{G}_t \right]
\end{align}
where $\proj_{\HH}[\cdot]$ denotes projection onto $\HH$. That is, \emph{anytime robust stochastic gradient descent}. These settings lead us to the following corollary.
\begin{cor}[Anytime robust SGD]\label{cor:anytime_sgd}
Consider $\algo$ implemented using (\ref{eqn:sgd_update}), with weights $\alpha_t = 1$ and $\beta_t \leq 1/\smooth$ for all $t \in [T]$. Then we have
\begin{align*}
\risk(\hbar_T) - \risk(\hstar) \leq \frac{2\diameter^{2}}{T\beta_T} + \max\left\{ 8\diameter\sigma\sqrt{\frac{2\log(\delta^{-1})}{T}}, \frac{12\smooth\diameter^{2}\log(\delta^{-1})}{T} \right\}
\end{align*}
with probability no less than $1-2\delta$.
\end{cor}

\subsection{Anytime robust AO-FTRL}

In this sub-section, we consider the case that $\algo$ is implemented using an adaptive optimistic follow-the-leader (AO-FTRL) procedure, namely updating as
\begin{align}\label{eqn:raoftrl_update}
h_{t+1} = \algo(h_t) \in \argmin_{h \in \HH} \left[ \alpha_{t+1} \langle \widetilde{G}_t, h \rangle + \sum_{i=1}^{t} \alpha_i \langle \overbar{G}_i, h \rangle + \sum_{i=0}^{t} \varphi_i(h) \right].
\end{align}
Here $(\varphi_t)$ is a sequence of regularizers that is now summed over for later notational convenience. Recalling the FTRL update (\ref{eqn:rftrl_update}), then clearly the AO-FTRL update is almost the same, save for the presence of $\widetilde{G}_t$ at each step $t$, with the interpretation is that it provides a prediction of the loss that will be incurred in the following step, i.e., $\widetilde{G}_t \approx \overbar{G}_{t+1}$.
\begin{thm}\label{thm:raoftrl_excess_risk}
Let Algorithm \ref{algo:anytime_robust} be run under the assumptions of Lemma \ref{lem:grad_error_highprob}, with $\algo$ implemented as in (\ref{eqn:raoftrl_update}), setting $\widetilde{G}_t = \overbar{G}_{t-1}$ for each $t > 1$. In addition, let each $\varphi_t(\cdot)$ be convex and non-negative, and denoting the regularizer partial sums as $\psi_t(\cdot) \defeq \sum_{i=0}^{t-1}\varphi_i(\cdot)$, let each $\psi_t$ be $\strong_t$-strongly convex, with weights set such that $(\smooth/\strong_t)\alpha_{t}^{2} \leq \alpha_{1:(t-1)}$ for $t>1$. Then, for any $\hstar \in \HH$ we have
\begin{align*}
\risk&(\hbar_T) - \risk(\hstar)\\
& \leq \frac{1}{\alpha_{1:T}} \left[ \frac{\alpha_{1}^{2}}{2\strong_1}\|\nabla\risk(\hbar_1)-\widetilde{G}_1\|_{\ast}^{2} + \sum_{t=1}^{T} \left[ \varphi_{t-1}(\hstar) - \varphi_{t-1}(h_t) \right] + 2\max\left\{ q_{\delta}(T), r_{\delta}(T) \right\} \right]
\end{align*}
with probability no less than $1-4\delta$.
\end{thm}
\noindent This theorem can be considered a robust, high-probability analogue of the results in expectation given by \citet[Thm.~3]{joulani2020a}. As such, it can be readily combined with existing regularization techniques \citep[Sec.~4]{joulani2020a} to achieve near-optimal rates under potentially heavy-tailed noise.

\section{Empirical analysis}\label{sec:empirical}

In this section we complement the preceding theoretical analysis with an application of the proposed learning strategy to real-world benchmark datasets. The practical utility of various gradient truncation mechanisms has already been well-studied in the literature  \citep{chen2017a,prasad2018a,lecue2018a,holland2019c}, and thus our chief point of interest here is if and when the feedback scheme utilized in Algorithm \ref{algo:anytime_robust} can outperform the traditional feedback mechanism given by (\ref{eqn:feedback_naive}), under a convex, differentiable true objective. Put more succinctly, the key question is: \textit{is there a practical benefit to querying at points with guarantees?}

\paragraph{Experimental setup}

At a high level, for each dataset of interest, we run multiple independent randomized trials, and for each trial, we run the methods of interest for multiple ``epochs'' (i.e., multiple passes over the data), recording the on-sample (training) and off-sample (testing) performance at the end of each epoch. As a simple and lucid example that implies a convex objective, we use multi-class logistic loss under a linear model; for a dataset with $k$ distinct classes, each predictor returns precisely $k$ scores which are computed as a linear combination of the input features. Thus with $k$ classes and $d_{\text{in}}$ input features, the total dimensionality is $d = kd_{\text{in}}$. For these experiments we run 10 independent trials. Everything is implemented by hand in Python (ver.~3.8), making significant use of the \texttt{numpy} library (ver.~1.20). For each method and each trial, the dataset is randomly shuffled before being split into training and testing subsets. If $n$ is the size of any given dataset, then the training set is of size $n_{\text{tr}} \defeq \lfloor 0.8n \rfloor$, and the test set is of size $n-n_{\text{tr}}$. Within each trial, for each epoch, the training data is also randomly shuffled. For all methods, the step size in update (\ref{eqn:sgd_update}) is fixed at $\beta_t = 2 / \sqrt{n_{\text{tr}}}$, for all steps $t$; this setting is appropriate for \texttt{Anytime-*} methods due to Corollary \ref{cor:anytime_sgd}, and also for \texttt{SGD-Ave} based on standard results such as \citet[Sec.~2.3]{nemirovski2009a}. The $(G_t)$ are obtained by direct computation of the logistic loss gradients, averaged over a mini-batch of size $8$; this size was set arbitrarily for speed and stability, and no other mini-batch values were tested. Furthermore, for each method and each trial, the initial value $h_1$ is randomly generated in a dimension-wise fashion from the uniform distribution on the interval $[-0.05,0.05]$. All raw input features are normalized to the unit interval $[0,1]$ in a per-feature fashion. We do not do any regularization, for any method being tested. We test three different learning procedures: averaged SGD using traditional feedback (\ref{eqn:feedback_naive}) (denoted \texttt{SGD-Ave}), anytime robust SGD precisely as in Algorithm \ref{algo:anytime_robust} and Corollary \ref{cor:anytime_sgd} (denoted \texttt{Anytime-Robust-SGD}), and finally anytime SGD without the robustification sub-routine \texttt{Process} (denoted \texttt{Anytime-SGD}).

\paragraph{Details for \texttt{Anytime-Robust-SGD}}

First, as a simple choice of anchors $\widetilde{h}$ and $\widetilde{g}$, we set $\widetilde{h}=h_1$ and estimate $\widetilde{g}$ using the empirical mean on the training data set; strictly speaking the proper approach is to split the dataset further and dedicate a subset to mean estimation, and the practitioner can easily refine this even further using any of the well-known robust high-dimensional mean estimators, see e.g.~\citet{lugosi2019a}. As for the thresholds $(\thres_t)$ used in the \texttt{Process} sub-routine, we set $\thres_t = \sqrt{n_{\text{tr}}/\log(\delta^{-1})}$ for all $t$, with a confidence level of $\delta = 0.05$ fixed throughout.

\paragraph{Dataset description}

We use eight datasets, identified respectively by the keywords: \texttt{adult},\footnote{\url{https://archive.ics.uci.edu/ml/datasets/Adult}} \texttt{cifar10},\footnote{\url{https://www.cs.toronto.edu/~kriz/cifar.html}} \texttt{cod\_rna},\footnote{\url{https://www.csie.ntu.edu.tw/~cjlin/libsvmtools/datasets/binary.html}} \texttt{covtype},\footnote{\url{https://archive.ics.uci.edu/ml/datasets/covertype}} \texttt{emnist\_balanced},\footnote{\url{https://www.nist.gov/itl/products-and-services/emnist-dataset}} \texttt{fashion\_mnist},\footnote{\url{https://github.com/zalandoresearch/fashion-mnist}} \texttt{mnist},\footnote{\url{http://yann.lecun.com/exdb/mnist/}} and \texttt{protein}.\footnote{\url{https://www.kdd.org/kdd-cup/view/kdd-cup-2004/Data}} See Table \ref{table:datasets} for a summary. Further background on all datasets is available at the URLs provided in the footnotes. Dataset size reflects the size after removal of instances with missing values, where applicable. For all datasets with categorical features, the ``input features'' given in Table \ref{table:datasets} represents the number of features after doing a one-hot encoding of all such features.

\begin{table*}[t!]
\begin{center}
\begin{tabular}{|l|l|l|l|l|}
\hline
Dataset & Size & Input features & Number of classes & Model dimension \\
\hline\hline
\texttt{adult} & 45,222 & 105 & 2 & 210\\
\hline
\texttt{cifar10} & 60,000 & 3,072 & 10 & 30,720\\
\hline
\texttt{cod\_rna} & 331,152 & 8 & 2 & 16\\
\hline
\texttt{covtype} & 581,012 & 54 & 7 & 378\\
\hline
\texttt{emnist\_balanced} & 131,600 & 784 & 47 & 36,848\\
\hline
\texttt{fashion\_mnist} & 70,000 & 784 & 10 & 7,840\\
\hline
\texttt{mnist} & 70,000 & 784 & 10 & 7,840\\
\hline
\texttt{protein} & 145,751 & 74 & 2 & 148\\
\hline
\end{tabular}
\end{center}
%\vspace{-0.5cm}
\caption{A summary of the benchmark datasets used for performance evaluation.}
\label{table:datasets}
\end{table*}

\paragraph{Software}

All code required to pre-process the data, run the experiments, and re-create the figures in this paper is available at: \url{https://github.com/feedbackward/anytime}

\paragraph{Results and discussion}

Our results are summarized in Figure \ref{fig:logistic_mean}, which plots the average training and test losses. For each trial, losses are averaged over datasets, and these average losses are themselves averaged over all trials to obtain the values plotted here. The impact of using feedback with guarantees is immediate; in all cases, we see a notable boost in learning efficiency. This positive effect holds essentially uniformly across the datasets used, with no hyperparameter tuning. For CIFAR-10, we observe that the robustified version performs worse than than vanilla anytime averaged SGD; this looks to be due to the simple $\widetilde{h}=h_1$ setting, and can be readily mitigated by updating $\widetilde{h}$ after one pass over the data. It is reasonable to conjecture that if we were to shift to more complex non-linear models, from the resulting lack of convexity in the objective, there might emerge a tradeoff between the stability encouraged by Algorithm \ref{algo:anytime_robust}, and the benefits of parameter space exploration that are incidental to the noisier gradients arising under (\ref{eqn:feedback_naive}).

\begin{figure*}[t!]
\centering
\includegraphics[width=0.65\textwidth]{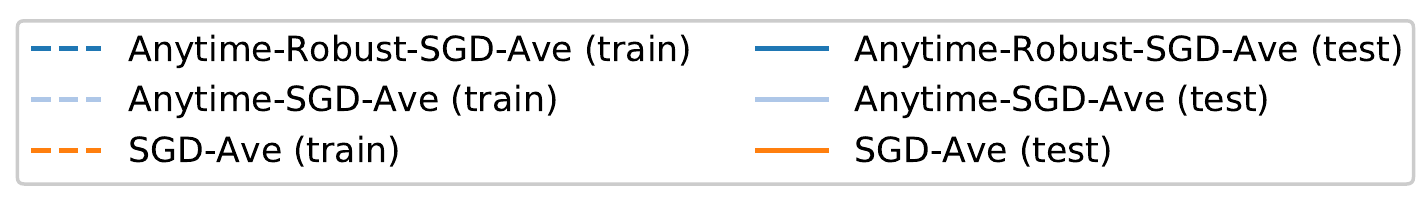}\\
\includegraphics[width=0.25\textwidth]{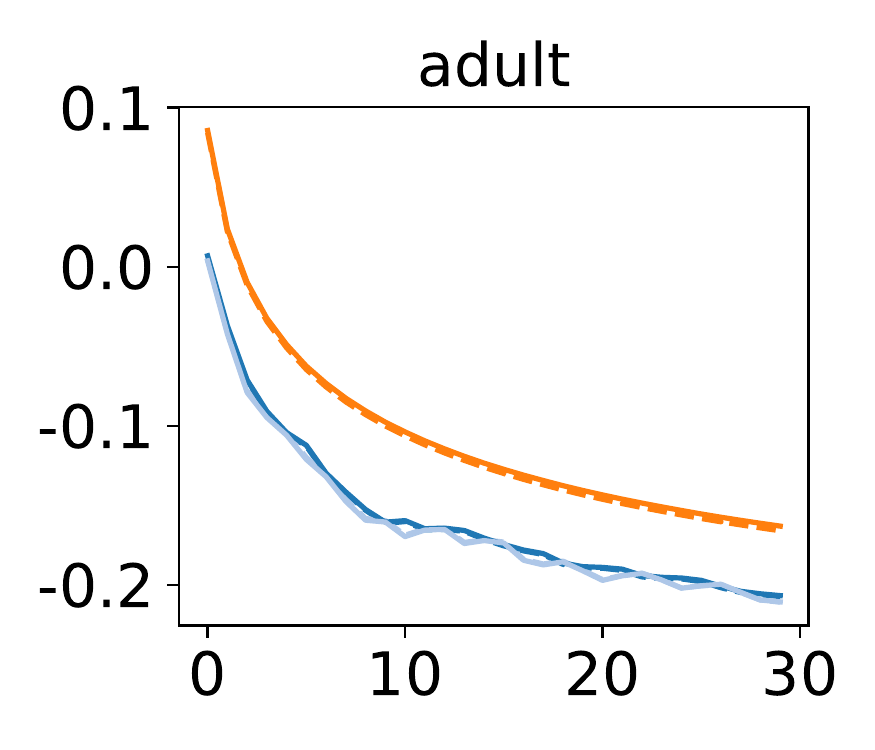}\,\includegraphics[width=0.25\textwidth]{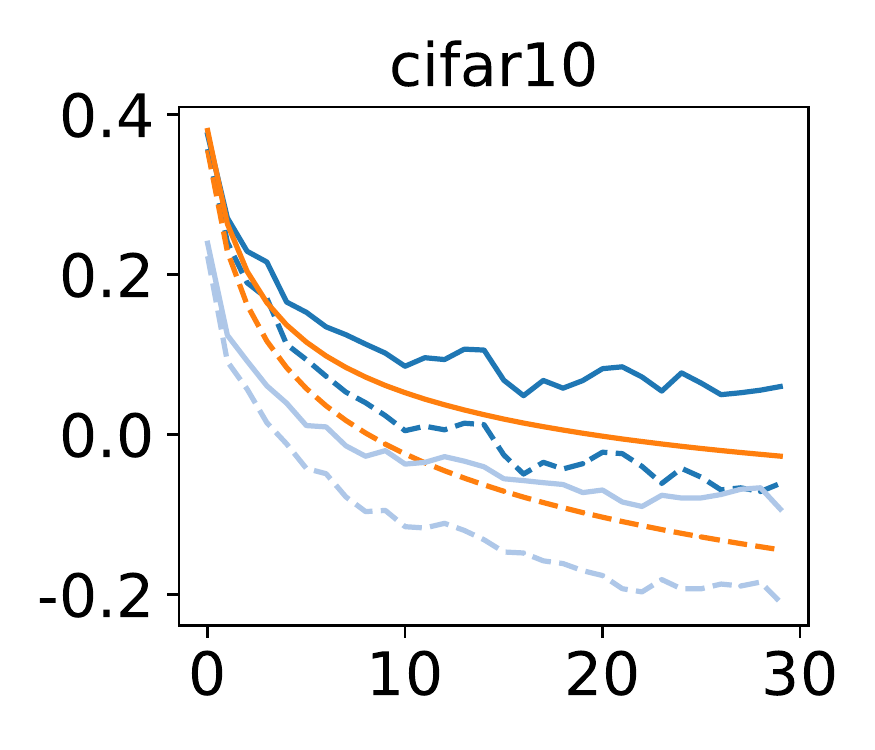}\,\includegraphics[width=0.25\textwidth]{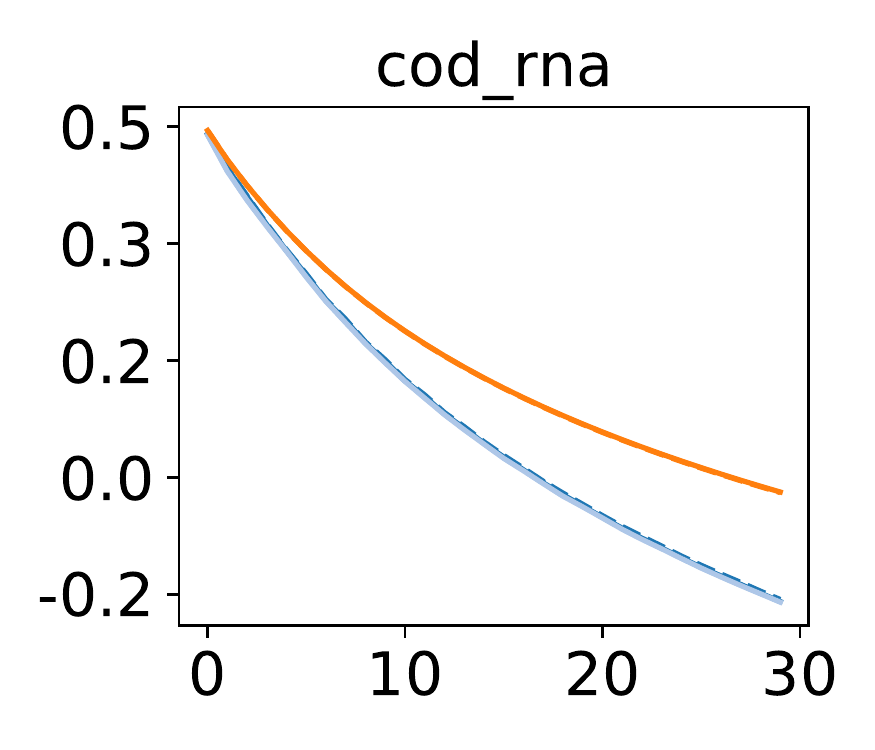}\,\includegraphics[width=0.25\textwidth]{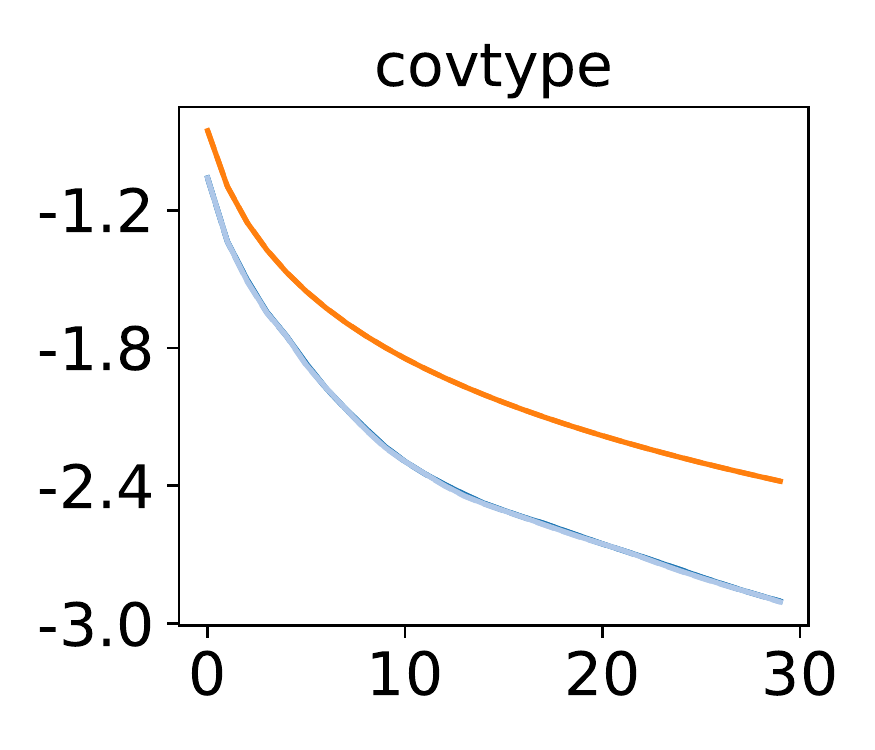}\\
\includegraphics[width=0.25\textwidth]{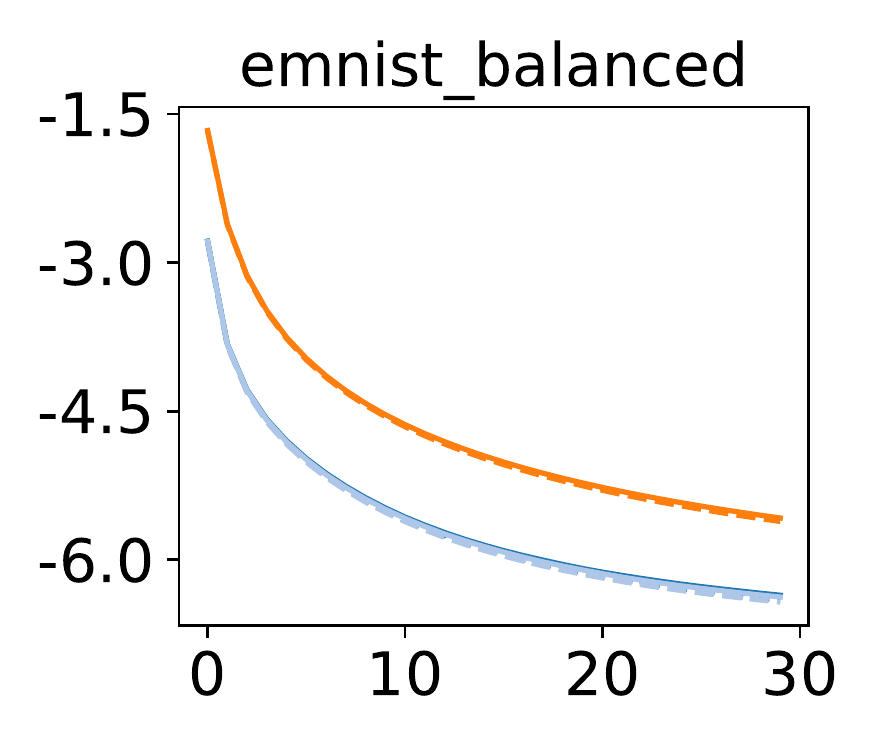}\,\includegraphics[width=0.25\textwidth]{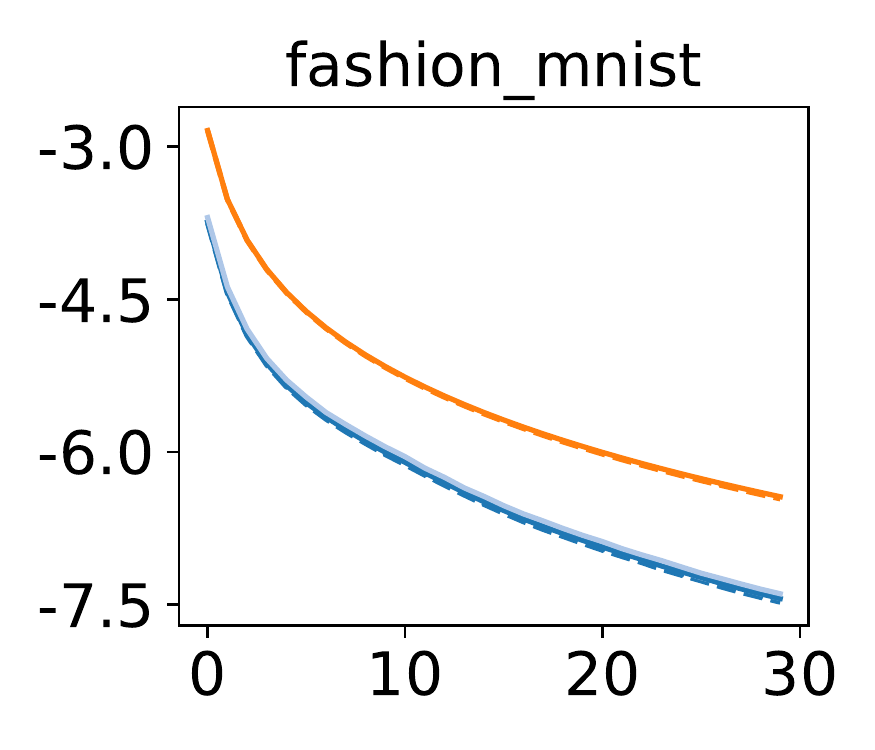}\,\includegraphics[width=0.25\textwidth]{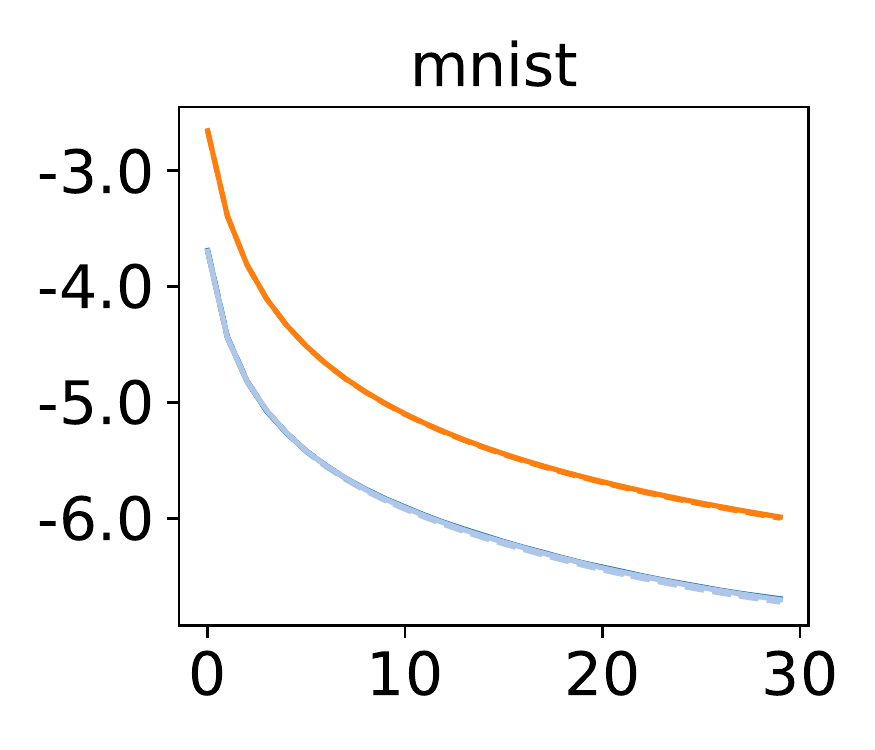}\,\includegraphics[width=0.25\textwidth]{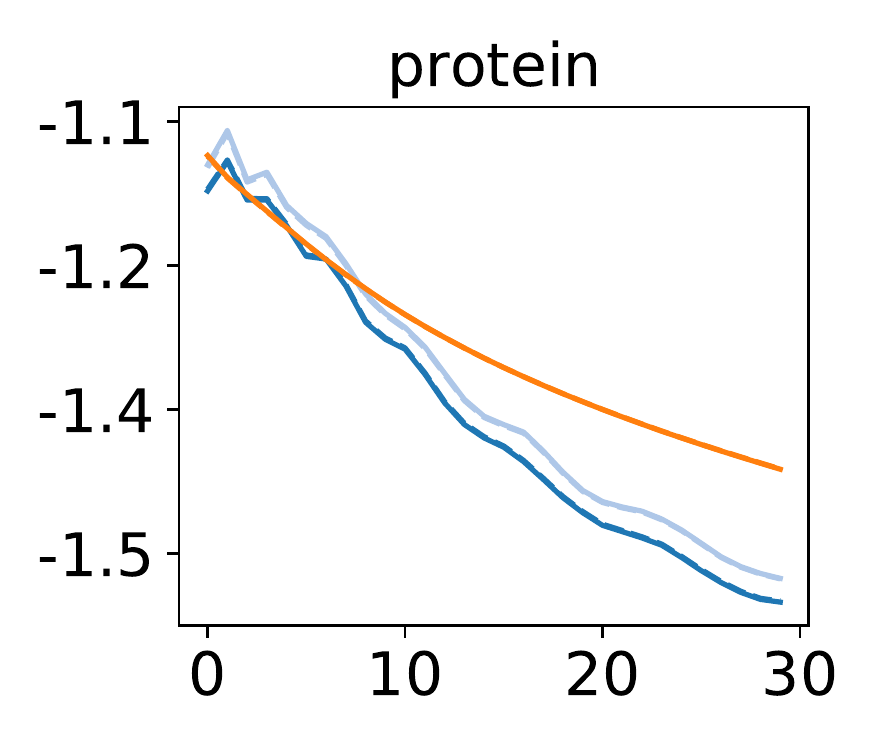}
\caption{Training and test loss versus epoch number, averaged over all trials, for each method. The eight individual plot titles correspond to dataset names.}
\label{fig:logistic_mean}
\end{figure*}

\section{Future directions}\label{sec:future}

From a technical perspective, the most salient direction moving forward is strengthening the robust estimation sub-routines to reduce the amount of prior knowledge required, and to potentially extend the methodology to cover non-smooth $\risk$. The requirement of a bounded domain can be removed (in the non-anytime setting) by using a more sophisticated update procedure \citep{gorbunov2020a}, and extending insights of this nature to refine or modify the procedure used to obtain Lemma \ref{lem:grad_error_highprob} is of natural interest. In most learning tasks of interest, the variance of the underlying feedback distribution may change significantly, and an adaptive strategy for setting $(c_t)$ is of interest both for strengthening formal guarantees and improving efficiency and stability in practice.

\appendix

\section{Technical appendix}

\subsection{Martingale concentration inequality}

Here we give a convenient form of Bernstein's inequality adapted to martingales; see \citet[Lem.~A.8]{cesa2006Games} for background and a concise proof.
\begin{lem}[Bernstein's inequality for martingales]\label{lem:bernstein}
Let $(V_t)$ be a martingale difference sequence with respect to the filtration $(\FF_t)$, bounded as $|V_t| \leq B$. Denote the partial sum of conditional variances by $\Sigma_{T}^{2} \defeq \sum_{t=1}^{T} \exx[ V_t \cond \FF_{t-1}]$. Under these assumptions, we have
\begin{align*}
\prr\left[ \left\{ \max_{t \in [T]} \sum_{i=1}^{t}V_i > \sqrt{2 \gamma_1 \gamma_2} + \frac{\sqrt{2}}{3}B\gamma_1 \right\} \bigcap \left\{\Sigma_{T}^{2} \leq \gamma_2 \right\} \right] \leq \exp(-\gamma_1)
\end{align*}
for any integer $T>0$ and real values $\gamma_1,\gamma_2 > 0$.
\end{lem}

\subsection{Convexity and smoothness}

We say that a function $f:\VV \to \overbar{\RR}$ is \term{convex} if for all $0 < \alpha < 1$ and $u,v \in \VV$, we have
\begin{align}\label{eqn:convexity}
f(\alpha u + (1-\alpha)v) \leq \alpha f(u) + (1-\alpha)f(v).
\end{align}
We say that $f$ is \term{strictly} convex if for all $0 < \alpha < 1$ and $u,v \in \VV$ such that $u \neq v$, we have
\begin{align}\label{eqn:convexity_strict}
f(\alpha u + (1-\alpha)v) < \alpha f(u) + (1-\alpha)f(v).
\end{align}
Finally, we say that $f$ is \term{$\strong$-strongly convex} (with respect to norm $\|\cdot\|$) if there exists some $\strong > 0$ such that for all $0 < \alpha < 1$ and $u,v \in \VV$, we have
\begin{align}
\label{eqn:convexity_strong}
f(\alpha u + (1-\alpha)v) \leq \alpha f(u) + (1-\alpha)f(v) - \frac{\strong\alpha(1-\alpha)}{2}\|u-v\|^{2}.
\end{align}
Strongly convex functions can be shown to satisfy a smoothness property with respect to the dual norm.
\begin{lem}[Duality of strong convexity and smoothness]\label{lem:sc_smooth_duality}
Let $f:\VV \to \overbar{\RR}$ be $\strong$-strongly convex with respect to $\|\cdot\|$. Then, for any $u,v \in \dom f$, we have
\begin{align*}
f(u) - f(v) \leq \langle \partial f(v), u-v \rangle + \frac{1}{2\strong}\|\partial f(u) - \partial f(v)\|_{\ast}^{2}.
\end{align*}
\end{lem}
This is a well-known fact based on duality relations between strong convexity and smoothness; see \citet{kakade2009b} and the references therein for background, or \citet[Lem.~7.5]{orabona2020a} for a more direct proof in the case of Euclidean inner product spaces (the result generalizes easily to arbitrary normed linear spaces).

\subsection{Mirror descent}\label{sec:ax_mirror}

In this sub-section, we prepare some technical results related to mirror descent procedures.\footnote{For some highly readable references, see for example \citet[Ch.~4, 6]{bubeck2015a} and \citet[Ch.~6]{orabona2020a}.} In what follows, our running assumptions will be that $(\VV,\|\cdot\|)$ is a real Banach space, $U \subset \VV$ is a non-empty, open, and convex set, and finally $\Phi:\VV \to \RR$ is strictly convex and Gateaux-differentiable on $U$.\footnote{This in particular implies that for any $u \in U$, the sub-differential of $\Phi$ at $u$ consists of a single element \citep[Prop.~2.40]{barbu2012ConvOptBanach}. We denote this by $\partial\Phi(u) = \{\nabla\Phi(u)\}$.} By $\breg: U \times U \to \RR$, we denote the \term{Bregman divergence} induced by $\Phi$, defined for any $u,v \in U$ as
\begin{align}
\breg(u;v) \defeq \Phi(u) - \Phi(v) - \langle \nabla\Phi(v), u - v \rangle.
\end{align}
From the strict convexity of $\Phi$, we have that $\breg(u;v) = 0$ if and only if $u=v$. The gradient of $\breg(\cdot;v)$ takes the simple form $\nabla\breg(u;v) = \nabla\Phi(u)-\nabla\Phi(v)$. For any $u,v,w \in U$, we have the following useful relation:\footnote{See \citet[Lem.~3.1]{chen1993a} for a simple proof and additional context.}
\begin{align}
\label{eqn:bregman_threepoint}
\langle \nabla\Phi(u)-\nabla\Phi(v) , u-w \rangle = \breg(u;v) + \breg(w;u) - \breg(w;v).
\end{align}
If $\Phi$ happens to be $\strong$-strongly convex, then
\begin{align}\label{eqn:bregman_sc}
\breg(u;v) \geq \frac{\strong}{2}\|u-v\|^{2}
\end{align}
for any choice of $u,v \in U$. The following lemma organizes several useful technical facts related to minimizers of functions derived from $\breg$.
\begin{lem}\label{lem:helper_md_general}
Set $U = \VV$, and let $f$ be any continuous linear functional on $\VV$. For any $v \in \VV$, consider the function $F_v:\VV \to \RR$ defined as $F_v(u) \defeq f(u) + \breg(u;v)$. Let $W \subset \VV$ be a closed and convex set. Assuming that $F_v$ takes a minimum value on both $\VV$ and $W$, we denote these respective minimizers by
\begin{align*}
\widetilde{v}^{\ast} \defeq \argmin_{w \in W} F_v(w), \qquad v^{\ast} \defeq \argmin_{u \in \VV} F_v(u).
\end{align*}
Under these assumptions, the following properties hold:
\begin{align}
\label{eqn:helper_md_general_res0}
\langle \nabla\Phi(v^{\ast}), u \rangle & = \langle \nabla\Phi(v)-\nabla f(w), u \rangle, \text{ for all } u,w \in \VV\\
\label{eqn:helper_md_general_res1}
\widetilde{v}^{\ast} & = \argmin_{w \in W} \breg(w;v^{\ast})\\
\label{eqn:helper_md_general_res2}
0 & \geq \langle \nabla\Phi(\widetilde{v}^{\ast})-\nabla\Phi(v), \widetilde{v}^{\ast}-w \rangle, \text{ for all } w \in W.
\end{align}
\end{lem}
The preceding lemma is quite general, but rests upon the assumption that $F_v(\cdot)$ takes its minimum on both $W$ and $\VV$. The following lemma treats the special case of (anytime robust) SMD, and gives conditions under which these minima are well-defined.
\begin{lem}\label{lem:helper_rsmd}
Let $\VV$ be a reflexive Banach space, with $\HH \subset \VV$ closed and bounded. Letting $(h_t)$ be the SMD sequence generated by (\ref{eqn:rsmd_update_proximal}), and defining
\begin{align*}
h_{t}^{\prime} \defeq \argmin_{h \in \VV} \left[ \langle \overbar{G}_t,h \rangle + \frac{1}{\beta_t}\breg(h;h_t) \right],
\end{align*}
both of these sequences are well-defined, and can be related by
\begin{align}
\label{eqn:rsmd_update_dual}
\nabla\Phi(h_t^{\prime}) & = \nabla\Phi(h_t) - \beta_t \overbar{G}_t\\
\label{eqn:rsmd_update_proj}
h_{t+1} & = \argmin_{h \in \HH} \breg(h;h_{t}^{\prime})
\end{align}
for all steps $t > 0$.
\end{lem}

\subsection{Detailed proofs}

\subsubsection{Proofs for gradient estimation error control}

\begin{proof}[Proof of Lemma \ref{lem:grad_error_highprob}]
The proof is composed of a few straightforward steps, with a flavour similar to \citet[Lems.~2--3]{nazin2019a}. First, we obtain bounds on the gradient errors and relevant moments. Second, we appeal to Bernstein's inequality for martingales, given for reference as Lemma \ref{lem:bernstein} in the appendix, to get a general-purpose high-probability bound. Finally, we just need to plug in a particular thresholding rule and clean up the resulting upper bounds. Before diving into the details, note that all the ancillary iterates used in Algorithm \ref{algo:anytime_robust} are such that $h_t \in \HH$. Convexity of $\HH$ thus implies that $\hbar_t \in \HH$ for all $t$. Since $\hstar \in \HH$ by assumption, all primal distance terms can thus be controlled using the finite diameter $\diameter$; this fact will be used repeatedly in this proof, without further mention. For readability, we will write $\epsilon_t \defeq \langle \overbar{G}_t - \nabla\risk(\hbar_t), h-h^{\prime} \rangle$ for any arbitrary choices of $h,h^{\prime} \in \HH$.

\paragraph{Step 1: basic bounds (uniform, mean, variance)}
To bound $|\epsilon_t|$ directly, from (\ref{eqn:grad_error_bd}) and the threshold setting (\ref{eqn:truncate_threshold_smooth}), it follows that for any $h,h^{\prime} \in \HH$ we have
\begin{align}
\nonumber
|\epsilon_t| & \leq \|\overbar{G}_t - \nabla\risk(\hbar_t)\|_{\ast}\|h-h^{\prime}\|\\
\nonumber
& \leq \|h-h^{\prime}\|\left[ 2\left(\widetilde{\varepsilon}\sigma + \smooth\|\widetilde{h}-\hbar_t\| \right) + \thres_0 \right]\\
\label{eqn:grad_error_highprob_1}
& \leq \diameter\left[ 2\left(\widetilde{\varepsilon}\sigma + \smooth\diameter \right) + \thres_0 \right].
\end{align}
For mean bounds, first observe that since  the sequence $(G_t)$ is unbiased in the sense of (\ref{eqn:grad_condition}), using the form (\ref{eqn:grad_est_handy_form}), we have
\begin{align*}
\exx_{[t-1]}\left[ \overbar{G}_t - \nabla\risk(\hbar_t) \right] = \exx_{[t-1]}(G_t - \widetilde{g})\idc_t,
\end{align*}
from which it follows that
\begin{align*}
\left\| \exx_{[t-1]}\left[ \overbar{G}_t - \nabla\risk(\hbar_t) \right] \right\|_{\ast} \leq \exx_{[t-1]}\|G_t - \nabla\risk(\hbar_t)\|_{\ast}\idc_t + \exx_{[t-1]}\|\nabla\risk(\hbar_t)-\widetilde{g}\|_{\ast}\idc_t.
\end{align*}
As a convenient bound on $\idc_t$, note that
\begin{align}
\nonumber
\exx_{[t-1]}\idc_t & \leq \exx_{[t-1]}\idc\{ \|G_t - \nabla\risk(\hbar_t)\|_{\ast} > \thres_0 \}\\
\label{eqn:indicator_bound}
& \leq \frac{\exx_{[t-1]}\|G_t - \nabla\risk(\hbar_t)\|_{\ast}^{2}}{\thres_0^{2}} \leq \frac{\sigma^{2}}{\thres_0^{2}}.
\end{align}
For completeness, a proof of (\ref{eqn:indicator_bound}) is given following the end of the current proof. With (\ref{eqn:indicator_bound}) in hand, using the inequalities of H\"{o}lder and Markov along with (\ref{eqn:grad_variance_bound}) and (\ref{eqn:grad_anchor_bound}), it follows that
\begin{align*}
\left\| \exx_{[t-1]}\left[ \overbar{G}_t - \nabla\risk(\hbar_t) \right] \right\|_{\ast} & \leq \sigma \sqrt{\exx_{[t-1]}\idc_t} + (\widetilde{\varepsilon}\sigma+\smooth\|\widetilde{h}-\hbar_t\|)\exx_{[t-1]}\idc_t\\
& \quad \leq \frac{\sigma^{2}}{\thres_0} + (\widetilde{\varepsilon}\sigma+\smooth\diameter)\left(\frac{\sigma}{\thres_0}\right)^{2}.
\end{align*}
Thus we have
\begin{align}
\nonumber
|\exx_{[t-1]}\epsilon_t| & \leq \left\|\exx_{[t-1]}\left[ \overbar{G}_t - \nabla\risk(\hbar_t) \right]\right\|_{\ast}\|h-h^{\prime}\|\\
\label{eqn:grad_error_highprob_2}
& \leq \diameter\left[ \frac{\sigma^{2}}{\thres_0} + (\widetilde{\varepsilon}\sigma+\smooth\diameter)\left(\frac{\sigma}{\thres_0}\right)^{2} \right].
\end{align}
Finally, for a rough conditional variance bound, first note that we have
\begin{align*}
\exx_{[t-1]}(\epsilon_t - \exx_{[t-1]}\epsilon_t)^{2} & \leq \exx_{[t-1]}|\epsilon_t|^{2}\\
& \leq \diameter^{2} \exx_{[t-1]}\|\overbar{G}_t - \nabla\risk(\hbar_t)\|_{\ast}^{2}.
\end{align*}
To get a more convenient bound, first note that the error vector can be written as a convex combination
\begin{align*}
\overbar{G}_t - \nabla\risk(\hbar_t) = (\widetilde{g}-\nabla\risk(\hbar_t))\idc_t + (G_t - \nabla\risk(\hbar_t))(1-\idc_t).
\end{align*}
Using the convexity of $\|\cdot\|_{\ast}^{2}$ along with (\ref{eqn:grad_variance_bound}), (\ref{eqn:grad_anchor_bound}), and (\ref{eqn:indicator_bound}), it follows that
\begin{align*}
\exx_{[t-1]}\|\overbar{G}_t - \nabla\risk(\hbar_t)\|_{\ast}^{2} & \leq \exx_{[t-1]}\|\widetilde{g}-\nabla\risk(\hbar_t)\|_{\ast}^{2}\idc_t + \exx_{[t-1]}\|G_t - \nabla\risk(\hbar_t)\|_{\ast}^{2}(1-\idc_t)\\
& \leq (\widetilde{\varepsilon} + \smooth\diameter)^{2}\left(\frac{\sigma}{\thres_0}\right)^{2} + \sigma^{2}.
\end{align*}
Taking this back to the rough variance bound given above, we obtain
\begin{align}
\nonumber
\exx_{[t-1]}(\epsilon_t - \exx_{[t-1]}\epsilon_t)^{2} & \leq \diameter^{2} \left[ (\widetilde{\varepsilon}\sigma + \smooth\diameter)^{2}\left(\frac{\sigma}{\thres_0}\right)^{2} + \sigma^{2} \right]\\
\label{eqn:grad_error_highprob_3}
& \leq \diameter^{2} \left[ (\widetilde{\varepsilon}\sigma + \smooth\diameter)\left(\frac{\sigma}{\thres_0}\right) + \sigma \right]^{2}.
\end{align}
This concludes the first step; the key take-aways are the bounds (\ref{eqn:grad_error_highprob_1}), (\ref{eqn:grad_error_highprob_2}), and (\ref{eqn:grad_error_highprob_3}).

\paragraph{Step 2: martingale concentration inequality}
Getting to the key quantities of interest, note that for any bound $|\exx_{[t-1]}\epsilon_t| \leq M$ that holds for all $t$ and any $\gamma > 0$, we have
\begin{align*}
\prr\left\{ \sum_{t=1}^{T}\alpha_t \epsilon_t > \alpha_{1:T}M + \gamma \right\} \leq \prr\left\{ \sum_{t=1}^{T}\alpha_t \left(\epsilon_t - \exx_{[t-1]}\epsilon_t\right) > \gamma \right\}.
\end{align*}
An obvious choice for $M$ is given by (\ref{eqn:grad_error_highprob_2}) in Step 1. This gives us a bounded martingale difference sequence to which we can apply Lemma \ref{lem:bernstein}. Using (\ref{eqn:grad_error_highprob_1}) and (\ref{eqn:grad_error_highprob_3}) from the previous step, the basic correspondence between the key elements of Lemma \ref{lem:bernstein} and our case is as follows:
\begin{align*}
V_t & \leftrightarrow \alpha_t (\epsilon_t - \exx_{[t-1]}\epsilon_t)\\
\Sigma_{T}^{2} & \leftrightarrow \sum_{t=1}^{T}\alpha_t^{2}\exx_{[t-1]}\left( \epsilon_t - \exx_{[t-1]}\epsilon_t  \right)^{2}\\
B & \leftrightarrow 2\diameter\left[ 2(\widetilde{\varepsilon}\sigma+\smooth\diameter) + \thres_0 \right] \left( \max_{t \in [T]} \alpha_t \right)\\
\gamma_1 & \leftrightarrow \log(\delta^{-1})\\
\gamma_2 & \leftrightarrow \diameter^{2} \left[ (\widetilde{\varepsilon}\sigma + \smooth\diameter)\left(\frac{\sigma}{\thres_0}\right) + \sigma \right]^{2}\left(\sum_{t=1}^{T}\alpha_t^{2}\right).
\end{align*}
Note the setting for $B$ just uses the simple bound $|\alpha_t\epsilon_t-\exx_{[t-1]}\alpha_t\epsilon_t| \leq 2|\alpha_t\epsilon_t|$ combined with (\ref{eqn:grad_error_highprob_1}). Using these correspondences to keep the notation clean, an application of Lemma \ref{lem:bernstein} and a union bound with the event of (\ref{eqn:truncate_anchors}) immediately implies that we have
\begin{align}\label{eqn:grad_error_highprob_4}
\sum_{t=1}^{T} \alpha_t \epsilon_t \leq \alpha_{1:T}M + \sqrt{2\gamma_2\log(\delta^{-1})} + \frac{\sqrt{2}\log(\delta^{-1})}{3}B
\end{align}
with probability no less than $1-2\delta$. General-purpose inequality (\ref{eqn:grad_error_highprob_4}) is the main take-away of this step.

\paragraph{Step 3: detailed calculations}
All that remains is to plug in the specific setting of the threshold parameter $\thres_0$ given in the lemma statement. Clearly, this can only take on two possible values, depending on whether inequality
\begin{align}\label{eqn:grad_error_highprob_5}
\smooth\diameter \leq \sigma \sqrt{T/\log(\delta^{-1})}
\end{align}
holds or not. Thus, there are just two cases to check. We take these one at a time and summarize the key intermediate steps.

First, assume (\ref{eqn:grad_error_highprob_5}) holds. Then $\thres_0 = \sigma \sqrt{T/\log(\delta^{-1})} + \widetilde{\varepsilon}\sigma$. Plugging this in, some elementary algebra cleans up the bounds such that the key constants in Step 2 can be taken as follows:
\begin{align*}
M & \leq 2\diameter\sigma\sqrt{\frac{\log(\delta^{-1})}{T}},\\
B & \leq 6\diameter\sigma\left( \sqrt{\frac{T}{\log(\delta^{-1})}} + \widetilde{\varepsilon}\sigma \right)\left(\max_{t \in [T]}\alpha_t\right),\\
\gamma_2 & \leq 4(\diameter\sigma)^{2}\left(\sum_{t=1}^{T}\alpha_t^{2}\right).
\end{align*}
Plugging these into (\ref{eqn:grad_error_highprob_4}), if we restrict $\delta$ such that $\log(\delta^{-1}) \leq T / (\widetilde{\varepsilon}\sigma)^{2}$ to obtain simpler bounds, tidying up via some elementary algebra leads to a bound of the form
\begin{align}
\label{eqn:grad_error_highprob_6}
\sum_{t=1}^{T} \alpha_t \epsilon_t \leq 2\diameter\sigma\sqrt{2\log(\delta^{-1})}\left[ \frac{\alpha_{1:T}}{\sqrt{T}} + \sqrt{\sum_{t=1}\alpha_t^{2}} + 2\left(\max_{t \in [T]}\alpha_t\right) \right]
\end{align}
for probability no less than $1-2\delta$.

Next, assume that (\ref{eqn:grad_error_highprob_5}) fails to hold, implying $\thres_0 = \smooth\diameter + \widetilde{\varepsilon}\sigma$. This also trivially implies $\sigma^{2} \leq (\smooth\diameter)^{2}\log(\delta^{-1})/T$. Using these facts, in an analogous fashion to the previous case, we have
\begin{align*}
M & \leq \frac{2\smooth\diameter^{2}\log(\delta^{-1})}{T},\\
B & \leq 6\diameter(\smooth\diameter+\widetilde{\varepsilon}\sigma)\left(\max_{t \in [T]}\alpha_t\right),\\
\gamma_2 & \leq 4\smooth^{2}\diameter^{4}\frac{\log(\delta^{-1})}{T}\left(\sum_{t=1}^{T}\alpha_t^{2}\right).
\end{align*}
Again substituting these bounds into (\ref{eqn:grad_error_highprob_4}), restricting $\delta$ such that $\log(\delta^{-1}) \leq T/\widetilde{\varepsilon}^{2}$ for simplicity and doing some cleanup, we have with probability at least $1-2\delta$ that
\begin{align}
\label{eqn:grad_error_highprob_7}
\sum_{t=1}^{T} \alpha_t \epsilon_t \leq 2\smooth\diameter^{2}\log(\delta^{-1}) \left[ \frac{\alpha_{1:T}}{T} + \sqrt{\frac{1}{T}\sum_{t=1}\alpha_t^{2}} + 2\sqrt{2}\left(\max_{t \in [T]}\alpha_t\right) \right].
\end{align}
The desired result follows immediately from the inequalities (\ref{eqn:grad_error_highprob_6}) and (\ref{eqn:grad_error_highprob_7}).
\end{proof}

\begin{proof}[Proof of inequality (\ref{eqn:indicator_bound})]
If $\idc_t = 0$, then trivially $\idc_t \leq \idc\{ \|G_t - \nabla\risk(\hbar_t)\|_{\ast} > \thres_0 \}$, so we consider the case of $\idc_t = 1$. With the threshold $\thres_t$ set as (\ref{eqn:truncate_threshold_smooth}) and using (\ref{eqn:grad_anchor_bound}), we have
\begin{align*}
\thres_0 & < \|G_t - \widetilde{g}\|_{\ast} - (\widetilde{\varepsilon}\sigma + \smooth\|\widetilde{h}-\hbar_t\|)\\
& \leq \|G_t - \widetilde{g}\|_{\ast} - \|\widetilde{g} - \nabla\risk(\hbar_t)\|_{\ast}\\
& \leq \|G_t - \nabla\risk(\hbar_t)\|_{\ast}.
\end{align*}
From this inequality, an application of Markov's inequality, and the variance bound (\ref{eqn:grad_variance_bound}), we have that (\ref{eqn:indicator_bound}) is valid.
\end{proof}

\subsubsection{Proofs of results in the main text}

\begin{proof}[Proof of Lemma \ref{lem:anytime_basic}]
A useful observation made by \citet{cutkosky2019a} is that using only the definition of $\hbar_t$ in (\ref{eqn:hbar_defn}), we have that for all integer $t > 0$ (and using the convention $\alpha_{1:0} \defeq 0$), the following relation holds:
\begin{align}\label{eqn:alpha_relation}
\alpha_t \left( \hbar_t - h_t \right) = \alpha_{1:(t-1)}\left(\hbar_{t-1} - \hbar_t\right).
\end{align}
This relation will be utilized shortly. To begin, direct manipulations show that we can write
\begin{align*}
\risk(\hbar_T) = \frac{1}{\alpha_{1:T}}\left[ \sum_{t=1}^{T} \alpha_t \risk(\hbar_t) + \sum_{t=2}^{T} \alpha_{1:(t-1)}\left[ \risk(\hbar_t) - \risk(\hbar_{t-1}) \right] \right].
\end{align*}
Subtracting $\risk(\hstar)$ from both sides, we have
\begin{align}\label{eqn:anytime_basic_1}
\risk(\hbar_T) - \risk(\hstar) = \frac{1}{\alpha_{1:T}}\left[ \sum_{t=1}^{T} \alpha_t \left[ \risk(\hbar_t) - \risk(\hstar) \right] + \sum_{t=2}^{T} \alpha_{1:(t-1)}\left[ \risk(\hbar_t) - \risk(\hbar_{t-1}) \right] \right].
\end{align}
For the first sum in (\ref{eqn:anytime_basic_1}), note that by definition of $\bregbase_{\risk}$, we have
\begin{align}\label{eqn:anytime_basic_2}
\sum_{t=1}^{T} \alpha_t \left[ \risk(\hbar_t) - \risk(\hstar) \right] = \sum_{t=1}^{T} \alpha_t \left[ \langle \partial\risk(\hbar_t), \hbar_t - \hstar \rangle - \bregbase_{\risk}(\hstar;\hbar_t) \right].
\end{align}
Similarly for the second sum in (\ref{eqn:anytime_basic_1}), we have
\begin{align}
\nonumber
\sum_{t=2}^{T} \alpha_{1:(t-1)}\left[ \risk(\hbar_t) - \risk(\hbar_{t-1}) \right] & = \sum_{t=2}^{T} \alpha_{1:(t-1)}\left[ \langle \partial\risk(\hbar_t), \hbar_t - \hbar_{t-1} \rangle - \bregbase_{\risk}(\hbar_{t-1};\hbar_t) \right]\\
\label{eqn:anytime_basic_3}
& = \sum_{t=2}^{T} \left[ \alpha_{t} \langle \partial\risk(\hbar_t), h_t - \hbar_t \rangle - \alpha_{1:(t-1)}\bregbase_{\risk}(\hbar_{t-1};\hbar_t) \right],
\end{align}
noting that the second equality is obtained by applying (\ref{eqn:alpha_relation}). The desired result follows from applying (\ref{eqn:anytime_basic_2})--(\ref{eqn:anytime_basic_3}) to (\ref{eqn:anytime_basic_1}), canceling terms and noting that $\hbar_1 = h_1$ by definition (\ref{eqn:hbar_defn}).
\end{proof}

\begin{proof}[Proof of Corollary \ref{cor:anytime_highprob}]
First noting simply that 
\begin{align*}
\sum_{t=1}^{T} \alpha_t \langle \nabla\risk(\hbar_t), h_t-\hstar \rangle = \sum_{t=1}^{T} \alpha_t \left[ \langle \overbar{G}_t, h_t-\hstar \rangle + \langle \nabla\risk(\hbar_t) - \overbar{G}_t, h_t-\hstar \rangle \right],
\end{align*}
the corollary follows from an application of Lemmas \ref{lem:anytime_basic} and \ref{lem:grad_error_highprob}, along with regret definition (\ref{eqn:regret_defn}), noting that we have set
\begin{align*}
B_T \defeq & \sum_{t=1}^{T}\alpha_t \bregbase_{\risk}(\hstar;\hbar_t) + \sum_{t=1}^{T-1}\alpha_{1:t}\bregbase_{\risk}(\hbar_t;\hbar_{t+1}),
\end{align*}
and $B_T \geq 0$ follows from the assumed convexity of $\risk$.
\end{proof}

\begin{proof}[Proof of Lemma \ref{lem:rftrl_regret}]
To emphasize how this result can be potentially generalized to other losses, write $\loss_t(h) \defeq \alpha_t \langle \overbar{G}_t, h \rangle$. The objective function used at each step is $f_{t}(h) \defeq \psi_{t}(h) + \sum_{i=1}^{t-1}\alpha_i \loss_i(h)$ for all $t > 1$, with $f_{1} \defeq \psi_{1}$. Thus, we have $h_{t} \in \argmin_{h \in \HH} f_{t}(h)$ for all $t \geq 1$.

First, note that for any $t \geq 1$, by trivial modifications, we have
\begin{align*}
-\sum_{i=1}^{t}\loss_i(\hstar) & = \psi_{t+1}(\hstar) - f_{t+1}(\hstar)\\
& = \psi_{t+1}(\hstar) - f_1(h_1) + f_1(h_1) - f_{t+1}(h_{t+1}) + f_{t+1}(h_{t+1}) - f_{t+1}(\hstar)\\
& = \psi_{t+1}(\hstar) - f_1(h_1) + \sum_{i=1}^{t}\left[ f_i(h_i) - f_{i+1}(h_{i+1}) \right] + f_{t+1}(h_{t+1}) - f_{t+1}(\hstar).
\end{align*}
By adding $\sum_{i=1}^{t}\loss_i(h)$ to both sides, we obtain a new regret expression:
\begin{align*}
\sum_{i=1}^{t}&\left[ \loss_i(h_i) - \loss_i(\hstar) \right]\\
& = \psi_{t+1}(\hstar) - f_1(h_1) + \sum_{i=1}^{t}\left[ f_i(h_i) - f_{i+1}(h_{i+1}) + \loss_i(h_i) \right] + f_{t+1}(h_{t+1}) - f_{t+1}(\hstar).
\end{align*}
The preceding equality is well-known \citep[Lem.~7.1]{orabona2020a}. Under the assumption that $\hstar \in \HH$, the optimality of $h_{t+1}$ with respect to $f_{t+1}$ implies $f_{t+1}(h_{t+1}) \leq f_{t+1}(\hstar)$, and thus we can immediately obtain a simpler inequality
\begin{align}
\label{eqn:rftrl_regret_basic}
\sum_{i=1}^{t}\left[ \loss_i(h_i) - \loss_i(\hstar) \right] \leq \psi_{t+1}(\hstar) - f_1(h_1) + \sum_{i=1}^{t}\left[ f_i(h_i) - f_{i+1}(h_{i+1}) + \loss_i(h_i) \right].
\end{align}
Moving forward, we will modify the bound in (\ref{eqn:rftrl_regret_basic}) in a way that is conducive to using our Lemma \ref{lem:grad_error_highprob}. Writing $\loss_{t}^{\ast}(h) \defeq \alpha_t \langle \partial\risk(\hbar_t), h \rangle$, note that we can manipulate the key summands as follows:
\begin{align}
\nonumber
f_i(h_i) & - f_{i+1}(h_{i+1}) + \loss_i(h_i)\\
\nonumber
& = (f_i(h_i) + \loss_{i}(h_i)) - (f_i(h_{i+1}) + \loss_{i}(h_{i+1})) + \psi_i(h_{i+1}) - \psi_{i+1}(h_{i+1})\\
\nonumber
& = (f_i(h_i) + \loss_{i}^{\ast}(h_i)) - (f_i(h_{i+1}) + \loss_{i}^{\ast}(h_{i+1})) + \psi_i(h_{i+1}) - \psi_{i+1}(h_{i+1})\\
\label{eqn:rftrl_intervention}
& \qquad + \left( \loss_{i}(h_i) - \loss_{i}^{\ast}(h_i) \right) - \left( \loss_{i}(h_{i+1}) - \loss_{i}^{\ast}(h_{i+1}) \right).
\end{align}
The first equality follows immediately from the definition of $f_{i+1}$, and the second equality just makes trivial modifications and re-arranges terms. Using our definitions of $\loss_i$ and $\loss_{i}^{\ast}$ here, the final two terms of (\ref{eqn:rftrl_intervention}) are simply
\begin{align}
\nonumber
\left( \loss_{i}(h_i) - \loss_{i}^{\ast}(h_i) \right) - \left( \loss_{i}(h_{i+1}) - \loss_{i}^{\ast}(h_{i+1}) \right) & = \alpha_i \langle \overbar{G}_i - \partial\risk(\hbar_i), h_i - h_{i+1} \rangle\\
\label{eqn:rftrl_intervention_cleanup_1}
& = \alpha_i \langle \partial\risk(\hbar_i) - \overbar{G}_i, h_{i+1} - h_i \rangle.
\end{align}
For the remaining terms in (\ref{eqn:rftrl_intervention}),  if we write $h_{i}^{\ast} \in \argmin_{h \in \VV} \left[ f_i(h) + \loss_{i}^{\ast}(h)\right]$, then
\begin{align}
\nonumber
(f_i(h_i) + \loss_{i}^{\ast}(h_i)) - (f_i(h_{i+1}) + \loss_{i}^{\ast}(h_{i+1})) & \leq (f_i(h_i) + \loss_{i}^{\ast}(h_i)) - (f_i(h_{i}^{\ast}) + \loss_{i}^{\ast}(h_{i}^{\ast}))\\
\label{eqn:rftrl_intervention_cleanup_2}
& \leq \frac{1}{2\strong_i}\|\partial\risk(\hbar_i)\|_{\ast}^{2}
\end{align}
where the second inequality follows from helper Lemma \ref{lem:sc_smooth_duality}, combined with our assumption that each $f_i$ is $\strong_i$-strongly convex.\footnote{Noting that $h_{i}^{\ast}$ is defined by minimizing a sub-differentiable function over the whole space $\VV$, and thus we have $0 \in \partial(f_i + \loss_{i}^{\ast})(h_{i}^{\ast})$, letting us simplify the upper bound given by Lemma \ref{lem:sc_smooth_duality}.} Applying (\ref{eqn:rftrl_intervention})--(\ref{eqn:rftrl_intervention_cleanup_2}) to the basic regret inequality (\ref{eqn:rftrl_regret_basic}), for any $t \geq 1$ we obtain the inequality
\begin{align*}
\sum_{i=1}^{t}\left[ \loss_i(h_i) - \loss_i(\hstar) \right] & \leq \psi_{t+1}(\hstar) - f_1(h_1) + \sum_{i=1}^{t}\left[ \psi_i(h_{i+1}) - \psi_{i+1}(h_{i+1}) \right]\\
& \qquad + \sum_{i=1}^{t}\left[ \frac{\|\partial\risk(\hbar_i)\|_{\ast}^{2}}{2\strong_i} + \alpha_i \langle \partial\risk(\hbar_i) - \overbar{G}_i, h_{i+1} - h_i \rangle \right].
\end{align*}
Finally, using the definition $f_1 \defeq \psi_1$ and setting $\psi_{T+1} \defeq \psi_{T}$ completes the proof.
\end{proof}

\begin{proof}[Proof of Lemma \ref{lem:rsmd_regret}]
To begin, note that using linearity and trivial modifications, we have
\begin{align*}
\beta_t & \langle \overbar{G}_t, h_t - \hstar \rangle\\
& = \beta_t \langle \overbar{G}_t, h_t - h_{t+1} \rangle + \langle \beta_t \overbar{G}_t \pm (\nabla\Phi(h_{t+1})-\nabla\Phi(h_t)), h_{t+1} - \hstar \rangle\\
& = \beta_t \langle \overbar{G}_t, h_t - h_{t+1} \rangle\\
& \qquad + \langle \nabla\Phi(h_{t+1})-\nabla\Phi(h_t), \hstar - h_{t+1} \rangle + \langle \nabla\Phi(h_t)-\nabla\Phi(h_{t+1}) - \beta_t \overbar{G}_t, \hstar - h_{t+1} \rangle.
\end{align*}
To clean this up, the last term can be bounded as
\begin{align*}
\langle \nabla\Phi(h_t)-\nabla\Phi(h_{t+1}) - \beta_t \overbar{G}_t, \hstar - h_{t+1} \rangle \leq 0,
\end{align*}
a fact which follows from applying Lemmas \ref{lem:helper_md_general}--\ref{lem:helper_rsmd} to (\ref{eqn:rsmd_update_proximal}), namely our modified SMD update.\footnote{Note that Lemma \ref{lem:helper_md_general} is very general and assumes the existence of solutions to minimization problems. This gap is filled in by Lemma \ref{lem:helper_rsmd}, for the special case of (\ref{eqn:rsmd_update_proximal}), proving that the required minimizers are well-defined.} Furthermore, if we re-write the second term on the right-hand side using the handy identity (\ref{eqn:bregman_threepoint}), we obtain the following inequality:
\begin{align}\label{eqn:rsmd_regret_1}
\beta_t \langle \overbar{G}_t, h_t - \hstar \rangle \leq \beta_t \langle \overbar{G}_t, h_t - h_{t+1} \rangle + \breg(\hstar;h_t) - \breg(\hstar;h_{t+1}) - \breg(h_{t+1};h_t).
\end{align}
To obtain a bound that lets us utilize the control offered by Lemma \ref{lem:grad_error_highprob}, here it is natural to control the first term on the right-hand side of (\ref{eqn:rsmd_regret_1}) by
\begin{align}
\nonumber
\beta_t \langle \overbar{G}_t, h_t - h_{t+1} \rangle & = \beta_t \left[ \langle \overbar{G}_t - \partial\risk(\hbar_t), h_t - h_{t+1} \rangle + \langle \partial\risk(\hbar_t), h_t - h_{t+1} \rangle \right]\\
\label{eqn:rsmd_regret_2}
& \leq \beta_t \langle \overbar{G}_t - \partial\risk(\hbar_t), h_t - h_{t+1} \rangle + \frac{\beta_{t}^{2}}{2z}\|\partial\risk(\hbar_t)\|_{\ast}^{2} + \frac{z}{2}\|h_t - h_{t+1}\|^{2},
\end{align}
where $z > 0$ is an arbitrary constant. To see that this inequality is valid, first note that the definition of the norm $\|\cdot\|_{\ast}$ implies $\langle \partial\risk(\hbar_t), h \rangle \leq \|\partial\risk(\hbar_t)\|_{\ast}\|h\|$ for any $h \in \VV$, and then simply apply the elementary inequality $2xy \leq zx^{2} + y^{2}/z$, valid for any $x,y \in \RR$ and $z > 0$.

Before combining inequalities (\ref{eqn:rsmd_regret_1}) and (\ref{eqn:rsmd_regret_2}), note that by helper inequality (\ref{eqn:bregman_sc}), we have that
\begin{align*}
\frac{z}{2}\|h_t - h_{t+1}\|^{2} - \breg(h_{t+1};h_t) \leq \frac{z-\strong}{2}\|h_t - h_{t+1}\|^{2},
\end{align*}
thus making $z = \strong$ an obvious choice for the free parameter $z$ in (\ref{eqn:rsmd_regret_2}). With this inequality in mind, combining (\ref{eqn:rsmd_regret_1}) and (\ref{eqn:rsmd_regret_2}), we obtain
\begin{align*}
\beta_t \langle \overbar{G}_t, h_t - \hstar \rangle \leq \beta_t \langle \overbar{G}_t - \partial\risk(\hbar_t), h_t - h_{t+1} \rangle + \frac{\beta_{t}^{2}}{2\strong}\|\partial\risk(\hbar_t)\|_{\ast}^{2} + \breg(\hstar;h_t) - \breg(\hstar;h_{t+1}).
\end{align*}
Dividing both sides by $\beta_t > 0$ and using linearity to re-arrange the gradient error terms yields the desired result.
\end{proof}

\begin{proof}[Proof of Theorem \ref{thm:anytime_rsmd}]
To start, we would like to use Lemma \ref{lem:rsmd_regret} to control $\regret(T;\algo)$ here. To do this, first note that whenever we have $(\beta_t/\beta_{t-1}) \leq (\alpha_t/\alpha_{t-1})$, the fact that $0 \leq \breg \leq \diameter_{\Phi}$ directly implies the following bound:
\begin{align}
\nonumber
\sum_{t=1}^{T} & \frac{\alpha_t}{\beta_t}\left[ \breg(\hstar;h_t)-\breg(\hstar;h_{t+1}) \right]\\
\nonumber
 & = \frac{\alpha_1}{\beta_1}\breg(\hstar;h_1)-\frac{\alpha_T}{\beta_T}\breg(\hstar;h_{T+1}) + \sum_{t=2}^{T} \breg(\hstar;h_t)\left(\frac{\alpha_t}{\beta_t}-\frac{\alpha_{t-1}}{\beta_{t-1}}\right)\\
\nonumber
& \leq \frac{\alpha_1}{\beta_1}\breg(\hstar;h_1) + \diameter_{\Phi} \sum_{t=2}^{T}\left(\frac{\alpha_t}{\beta_t}-\frac{\alpha_{t-1}}{\beta_{t-1}}\right)\\
\nonumber
& = \frac{\alpha_1}{\beta_1}\breg(\hstar;h_1) + \diameter_{\Phi} \left(\frac{\alpha_T}{\beta_T}-\frac{\alpha_1}{\beta_1}\right)\\
\label{eqn:anytime_rsmd_0}
& \leq \frac{\alpha_T}{\beta_T}\diameter_{\Phi}.
\end{align}
Thus, applying (\ref{eqn:anytime_rsmd_0}) to Lemma \ref{lem:rsmd_regret}, we have
\begin{align}\label{eqn:anytime_rsmd_1}
\regret(T;\algo) \leq \frac{\alpha_T}{\beta_T}\diameter_{\Phi} + \sum_{t=1}^{T} \left[ \frac{\alpha_t \beta_t}{2\strong} \|\nabla\risk(\hbar_t)\|_{\ast}^{2} + \alpha_t \langle \nabla\risk(\hbar_t) - \overbar{G}_t, h_{t+1} - h_t \rangle \right].
\end{align}
Using the stationarity of $\hstar \in \HH$ and the assumption that $\beta_t \leq \strong/\smooth$, we have
\begin{align}
\nonumber
\frac{\alpha_t \beta_t}{2\strong}\|\nabla\risk(\hbar_t)\|_{\ast}^{2} = \frac{\alpha_t \beta_t}{2\strong}\|\nabla\risk(\hbar_t) - \nabla\risk(\hstar)\|_{\ast}^{2} & \leq \frac{\alpha_t}{2\smooth}\|\nabla\risk(\hbar_t) - \nabla\risk(\hstar)\|_{\ast}^{2}\\
\label{eqn:anytime_rsmd_2}
& \leq \alpha_t \bregbase_{\risk}(\hstar;\hbar_t)
\end{align}
noting that the final inequality uses a well-known characterization of $\smooth$-smoothness \citep[Thm.~2.1.5]{nesterov2018ConvOptNew}.\footnote{While the cited result of \citet{nesterov2018ConvOptNew} is for Euclidean space, the argument readily generalizes to normed linear spaces.} Applying (\ref{eqn:anytime_rsmd_2}) to (\ref{eqn:anytime_rsmd_1}) to get a new regret bound, and subsequently applying that new regret bound to the key starting expression (\ref{eqn:regret_starting_point}), since the both the $\langle \overbar{G}_t - \nabla\risk(\hbar_t), h_t \rangle$ and $\alpha_t \bregbase_{\risk}(\hstar;\hbar_t)$ terms cancel out, and convexity of $\risk$ implies $\bregbase_{\risk} \geq 0$, we obtain
\begin{align*}
\risk(\hbar_T) - \risk(\hstar) \leq \frac{1}{\alpha_{1:T}} \left[ \frac{\alpha_T}{\beta_T}\diameter_{\Phi} + \sum_{t=1}^{T} \alpha_t \langle \overbar{G}_t - \nabla\risk(\hbar_t), \hstar - h_{t+1} \rangle \right].
\end{align*}
Since $h_{t+1}, \hstar \in \HH$, applying Lemma \ref{lem:grad_error_highprob} yields the desired result.
\end{proof}

\begin{proof}[Proof of Corollary \ref{cor:anytime_sgd}]
Direct calculations shows that when $\alpha_t = 1$ for all $t$, we have
\begin{align*}
q_{\delta}(T) & = 4\diameter\sigma\sqrt{2\log(\delta^{-1})}(\sqrt{T}+1) \leq 8\diameter\sigma\sqrt{2T\log(\delta^{-1})}\\
r_{\delta}(T) & = 4\smooth\diameter^{2}\log(\delta^{-1})\left(1 + \sqrt{2}\right) \leq 12\smooth\diameter^{2}\log(\delta^{-1}).
\end{align*}
Furthermore, the SGD update (\ref{eqn:sgd_update}) amounts to the setup of Theorem specialized to $\strong = 1$ and $\Phi(u) = \|u\|_{2}^{2}/2$, and it thus follows immediately that $\diameter_{\Phi} = 2\diameter^{2}$, and thus the upper bound can be written simply as
\begin{align*}
\risk(\hbar_T) - \risk(\hstar) \leq \frac{2\diameter^{2}}{T\beta_T} + \frac{\max\{q_{\delta}(T),r_{\delta}(T)\}}{T}.
\end{align*}
Plugging in the bounds given above for $q_{\delta}(T)$ and $r_{\delta}(T)$ immediately yields the desired result.
\end{proof}

\begin{proof}[Proof of Theorem \ref{thm:raoftrl_excess_risk}]
First, note that the existence of the sequence $(h_t)$ holds from the same argument as made in the proof of Lemma \ref{lem:helper_rsmd}. We break the argument into two straightforward steps.

\paragraph{Step 1: high-probability regret control for AO-FTRL}

As a foundational fact upon which we initiate our analysis, from \citet[Thm.~17]{joulani2020a}, for $\algo$ implemented by (\ref{eqn:raoftrl_update}), the regret defined by (\ref{eqn:regret_defn}) can be bounded above as
\begin{align}
\label{eqn:raoftrl_0}
\regret(T;\algo) \leq \sum_{t=1}^{T}\left[ \varphi_{t-1}(\hstar) - \varphi_{t-1}(h_t) - \bregbase_{\psi_t}(h_{t+1};h_t)  + \alpha_t \langle \overbar{G}_t - \widetilde{G}_t, h_t - h_{t+1} \rangle \right].
\end{align}
To control the gradient error term, for each $t \geq 1$, we break down the differences as
\begin{align}
\label{eqn:raoftrl_breakdown}
\overbar{G}_t - \widetilde{G}_t = \left[ \overbar{G}_t - \nabla\risk(\hbar_t) \right] + \left[ \nabla\risk(\hbar_t) - \nabla\risk(\hbar_{t-1}) \right] + \left[ \nabla\risk(\hbar_{t-1}) - \widetilde{G}_t \right].
\end{align}
We will need to consider each of these three terms plugged into $\alpha_t \langle \cdot , h_t - h_{t+1} \rangle$, one at a time. Since the first term on the right-hand side of (\ref{eqn:raoftrl_breakdown}) also appears in (\ref{eqn:regret_starting_point}), we leave it as-is for now, to be canceled out later. We will next consider both of the remaining terms for the case of $t>1$, before coming back to the $t=1$ terms. Under $t>1$, for the second term in (\ref{eqn:raoftrl_breakdown}), noting that $\psi_t$ is $\strong_t$-strongly convex by assumption, we have
\begin{align*}
& \alpha_t \langle \nabla\risk(\hbar_t) - \nabla\risk(\hbar_{t-1}), h_t - h_{t+1} \rangle - \bregbase_{\psi_t}(h_{t+1};h_t)\\
& \qquad \leq \alpha_t \langle \nabla\risk(\hbar_t) - \nabla\risk(\hbar_{t-1}), h_t - h_{t+1} \rangle - \frac{\strong_t}{2}\|h_t - h_{t+1}\|^{2}\\
& \qquad \leq \frac{\alpha_{t}^{2}}{2\strong_t}\| \nabla\risk(\hbar_t) - \nabla\risk(\hbar_{t-1}) \|_{\ast}^{2},
\end{align*}
noting that the final inequality follows by the Fenchel-Young inequality, in particular the fact that $\langle g^{\ast}, h \rangle - (a/2)\|h\|^{2} \leq (1/(2a))\|g^{\ast}\|_{\ast}^{2}$, for any $h \in \VV$, $g^{\ast} \in \VV^{\ast}$, $a > 0$. Since we have assumed that under a $\smooth$-smooth $\risk$, weights are set such that $(\smooth/\strong_t)\alpha_{t}^{2} \leq \alpha_{1:(t-1)}$, we have
\begin{align*}
\frac{\alpha_{t}^{2}}{2\strong_t}\| \nabla\risk(\hbar_t) - \nabla\risk(\hbar_{t-1}) \|_{\ast}^{2} \leq \frac{\alpha_{1:(t-1)}}{2\smooth}\| \nabla\risk(\hbar_t) - \nabla\risk(\hbar_{t-1}) \|_{\ast}^{2} \leq \alpha_{1:(t-1)} \bregbase_{\risk}(\hbar_{t-1};\hbar_t)
\end{align*}
for each $t > 1$, where the last step uses a basic characterization of $\smooth$-smoothness \citep[Thm.~2.1.5]{nesterov2018ConvOptNew}. As such, for each $t > 1$ we have
\begin{align}\label{eqn:raoftrl_3}
\alpha_t \langle \nabla\risk(\hbar_t) - \nabla\risk(\hbar_{t-1}), h_t - h_{t+1} \rangle - \bregbase_{\psi_t}(h_{t+1};h_t) \leq \alpha_{1:(t-1)} \bregbase_{\risk}(\hbar_{t-1};\hbar_t).
\end{align}
For the third term in (\ref{eqn:raoftrl_breakdown}), Since we are setting $\widetilde{G}_t = \overbar{G}_{t-1}$ for each $t>1$, we can bound this as
\begin{align}
\nonumber
\sum_{t=2}^{T} \alpha_t \langle \nabla\risk(\hbar_{t-1}) - \widetilde{G}_t, h_t - h_{t+1} \rangle & = \sum_{t=2}^{T} \alpha_t \langle \nabla\risk(\hbar_{t-1}) - \overbar{G}_{t-1}, h_t - h_{t+1} \rangle\\
\label{eqn:raoftrl_4a}
& \leq \sum_{t=2}^{T} \alpha_t \sup_{h,h^{\prime} \in \HH} \langle \nabla\risk(\hbar_{t-1}) - \overbar{G}_{t-1}, h - h^{\prime} \rangle\\
\label{eqn:raoftrl_4b}
& \leq \max\left\{ q_{\delta}(T), r_{\delta}(T) \right\},
\end{align}
where $q_{\delta}(\cdot)$ and $r_{\delta}(\cdot)$ are as defined in Lemma \ref{lem:grad_error_highprob}, and this bound holds with probability no less than $1-2\delta$. While (\ref{eqn:raoftrl_4b}) does not follow immediately from Lemma \ref{lem:grad_error_highprob} as-is, it follows from a straightforward modified argument, which we outline immediately after the conclusion of this proof. Finally, to cover the $t=1$ terms, just as before we have
\begin{align}
\nonumber
\alpha_1 & \left[ \langle \nabla\risk(\hbar_1)-\nabla\risk(\hbar_0), h_1 - h_2 \rangle + \langle \nabla\risk(\hbar_0) - \widetilde{G}_1, h_1 - h_2 \rangle \right] - \bregbase_{\psi_1}(h_2;h_1)\\
\nonumber
& = \alpha_1 \left[ \langle \nabla\risk(\hbar_1)-\widetilde{G}_1, h_1 - h_2 \rangle \right] - \bregbase_{\psi_1}(h_2;h_1)\\
\label{eqn:raoftrl_4c}
& \leq \frac{\alpha_{1}^{2}}{2\strong_1}\|\nabla\risk(\hbar_1)-\widetilde{G}_1\|_{\ast}^{2}.
\end{align}
This term can now be left as-is. The main takeaways of this step are the bounds (\ref{eqn:raoftrl_3}), (\ref{eqn:raoftrl_4b}), and (\ref{eqn:raoftrl_4c}).

\paragraph{Step 2: cleanup}

In consideration of the regret bound (\ref{eqn:raoftrl_0}), the breakdown (\ref{eqn:raoftrl_breakdown}), and the subsequent upper bounds (\ref{eqn:raoftrl_3}), (\ref{eqn:raoftrl_4b}), and (\ref{eqn:raoftrl_4c}), we obtain a new bound of the form
\begin{align}
\nonumber
\alpha_{1:T} & \left[ \risk(\hbar_T) - \risk(\hstar) \right]\\
\nonumber
& \leq \frac{\alpha_{1}^{2}}{2\strong_1}\|\nabla\risk(\hbar_1)-\widetilde{G}_1\|_{\ast}^{2} + \max\left\{ q_{\delta}(T), r_{\delta}(T) \right\}\\
\label{eqn:raoftrl_5}
& \qquad + \sum_{t=1}^{T} \left[ \varphi_{t-1}(\hstar) - \varphi_{t-1}(h_t) + \alpha_t \langle \overbar{G}_t - \nabla\risk(\hbar_t), \hstar - h_{t+1} \rangle \right].
\end{align}
To see that this is valid, first be careful to note that the last sum is in terms of $\hstar - h_{t+1}$ and not $\hstar-h_t$; this is due to cancellation of $\langle \cdot , h_t - h_{t+1} \rangle$ terms in the regret bound (via (\ref{eqn:raoftrl_0}) and (\ref{eqn:raoftrl_breakdown}). Also note that the sum of $T-1$ Bregman divergence terms cancels out, recalling the form of (\ref{eqn:regret_starting_point}) and the fact that
\begin{align*}
\sum_{t=2}^{T} \alpha_{1:(t-1)} \bregbase_{\risk}(\hbar_{t-1};\hbar_t) = \sum_{t=1}^{T-1} \alpha_{1:t} \bregbase_{\risk}(\hbar_t;\hbar_{t+1}).
\end{align*}
As for the remaining $T$ Bregman terms being subtracted, just bound them as $-\bregbase_{\risk}(\hstar;\hbar_t) \leq 0$ for each $t \geq 1$. Since $h_{t+1}, \hstar \in \HH$, we can directly apply Lemma \ref{lem:grad_error_highprob} to (\ref{eqn:raoftrl_5}) to get
\begin{align*}
\alpha_{1:T} & \left[ \risk(\hbar_T) - \risk(\hstar) \right]\\
& \leq \frac{\alpha_{1}^{2}}{2\strong_1}\|\nabla\risk(\hbar_1)-\widetilde{G}_1\|_{\ast}^{2} + 2\max\left\{ q_{\delta}(T), r_{\delta}(T) \right\} + \sum_{t=1}^{T} \left[ \varphi_{t-1}(\hstar) - \varphi_{t-1}(h_t) \right]
\end{align*}
with probability no less than $1-4\delta$, having taken a union bound over the two good events of interest. Dividing both sides by $\alpha_{1:T}$ yields the desired result.
\end{proof}

\begin{proof}[Proof of inequality (\ref{eqn:raoftrl_4b})]
A quick glance at the proof of Lemma \ref{lem:grad_error_highprob} shows that its statement can be easily generalized as follows: leaving the $(\alpha_t)$ used for constructing $(\hbar_t)$ as-is, we can replace the $(\alpha_t)$ weights used in the sum being bounded (in the lemma statement) by an arbitrary sequence $(a_t)$ unrelated to $(\hbar_t)$, and as long as $\exx_{[t-1]} a_t = a_t$ almost surely (for each $t$), the same concentration holds, with $(\alpha_t)$ in the bounds replaced by $(a_t)$. To apply this fact to control the sum in (\ref{eqn:raoftrl_4a}), instead of requiring that $\exx_{[t-1]}\alpha_t = \alpha_t$ almost surely, we need to go one step earlier and require $\exx_{[t-2]} \alpha_t = \alpha_t$ almost surely, for each $t \in [T]$. Then using the fact that $\alpha_t \geq 0$ for all $t$, the desired high-probability bound (\ref{eqn:raoftrl_4b}) follows easily.
\end{proof}

\subsubsection{Proofs for mirror descent helper results}

\begin{proof}[Proof of Lemma \ref{lem:helper_md_general}]
To start, note that by direct inspection one can easily verify that the strict convexity and propriety of $\Phi$ imply that $F_v(\cdot)$ is also strictly convex and proper. As such, since we are assuming that $F_v(\cdot)$ achieves its minimum on $\VV$ and $W$, strict convexity implies that the minimizers must be unique, thereby justifying the definition of $\widetilde{v}^{\ast}$ and $v^{\ast}$ given in the lemma statement.

By assumption $F_v(\cdot)$ is finite and continuous, and thus sub-differentiable on $\VV$.\footnote{\citet[Prop.~2.36]{barbu2012ConvOptBanach}.} Writing $F_v^{\ast} \defeq \inf\{F_v(u): u \in \VV\}$, recall the following global optimality characterization:\footnote{\citet[Sec.~2.2.1]{barbu2012ConvOptBanach}.}
\begin{align*}
F_v(v^{\ast}) = F_v^{\ast} \iff 0 \in \partial F_v(v^{\ast}).
\end{align*}
Since $f$ is trivially Gateaux-differentiable, it follows that $F_v(\cdot)$ is differentiable, and thus $\partial F_v(u) = \{\nabla F_v(u)\}$ for all $u \in \VV$.\footnote{Having a single sub-gradient essentially characterizes Gateaux-differentiability \citep[Prop.~2.40]{barbu2012ConvOptBanach}.} It thus follows that
\begin{align}\label{eqn:helper_md_general_1}
\nabla F_v(v^{\ast}) = \nabla f(v^{\ast}) + \nabla\Phi(v^{\ast}) - \nabla\Phi(v) = 0.
\end{align}
This observation immediately implies that for all $u,w \in \VV$ we have
\begin{align*}
\langle \nabla\Phi(v^{\ast}), u \rangle = \langle \nabla\Phi(v)-\nabla f(w), u \rangle.
\end{align*}
This holds because of the basic fact that since $f$ is a linear functional, we have $\langle \nabla f(w), u \rangle = f(u)$ for all $u,w \in \VV$.\footnote{See for example \citet[Sec.~7.2]{luenberger1969Book}.} This proves the first desired result (\ref{eqn:helper_md_general_res0}).

Moving forward with (\ref{eqn:helper_md_general_1}) in hand, note that for any $u \in \VV$, the Bregman divergence of $u$ from $v^{\ast}$ can be equivalently expressed using both $\Phi$ and $F_v$:
\begin{align}
\nonumber
\breg(u;v^{\ast}) & = \breg(u;v^{\ast}) + \langle \nabla F_v(v^{\ast}), u-v^{\ast} \rangle\\
\nonumber
& = \langle \nabla f(v^{\ast}), u-v^{\ast} \rangle + \Phi(u) - \Phi(v^{\ast}) - \langle \nabla\Phi(v), u-v^{\ast} \rangle\\
\nonumber
& = f(u) - f(v^{\ast}) + \breg(u;v) - \breg(v^{\ast};v)\\
\nonumber
& = F_v(u) - F_v(v^{\ast})\\
\label{eqn:helper_md_general_2}
& = \bregbase_{F_v}(u;v^{\ast}).
\end{align}
The first two equalities follow trivially from (\ref{eqn:helper_md_general_1}) and the definition of $\breg$. The third equality follows from the definition of $\breg$, the linearity of $f$, and the fact that $\langle \nabla f(w), u \rangle = f(u)$ for all $u,w \in \VV$, as mentioned above. The final two equalities just use the definition of $F_v$ and Bregman divergences, along with (\ref{eqn:helper_md_general_1}) again. With the identity (\ref{eqn:helper_md_general_2}) in mind, write the minimizer of this function over $W$ as
\begin{align*}
v^{\prime} \defeq \argmin_{u \in W} \breg(u;v^{\ast}).
\end{align*}
We then can readily confirm that
\begin{align*}
F_v(v^{\prime}) - F_v(v^{\ast}) & = \breg(v^{\prime};v^{\ast})\\
& \leq \breg(\widetilde{v}^{\ast};v^{\ast}) = F_v(\widetilde{v}^{\ast}) - F_v(v^{\ast}).
\end{align*}
The inequality follows from the optimality given in the definition of $v^{\prime}$, and the two equalities follow from (\ref{eqn:helper_md_general_2}) above. This clearly implies $F_v(v^{\prime}) \leq F_v(\widetilde{v}^{\ast})$, but recall that the definition of $\widetilde{v}^{\ast}$ tells us that $F_v(v^{\prime}) \geq F_v(\widetilde{v}^{\ast})$ also must hold. As such, $F_v(v^{\prime}) = F_v(\widetilde{v}^{\ast})$, and the strict convexity of $F_v$ implies that $\widetilde{v}^{\ast} = v^{\prime}$, and thus the second desired result (\ref{eqn:helper_md_general_res1}) linking up $\widetilde{v}^{\ast}$ and $v^{\ast}$ is obtained. Utilizing this new characterization of $\widetilde{v}^{\ast}$, we can immediately observe
\begin{align*}
\langle \nabla\Phi(\widetilde{v}^{\ast})-\nabla\Phi(v), \widetilde{v}^{\ast}-w \rangle & = (-1)\langle \nabla\breg(\widetilde{v}^{\ast};v), w-\widetilde{v}^{\ast} \rangle\\
& \leq 0,
\end{align*}
noting that the inequality follows from the fact that $\langle \nabla\breg(\widetilde{v}^{\ast};v), w-\widetilde{v}^{\ast} \rangle \geq 0$ for all $w \in W$, a standard optimality condition.\footnote{See \citep[Prop.~3.1.4]{bertsekas2015ConvexOpt} for Euclidean spaces, and \citet[Sec.~7.4, Thm.~2]{luenberger1969Book} for general vector spaces.} This gives us (\ref{eqn:helper_md_general_res2}) and thus concludes the proof.
\end{proof}

\begin{proof}[Proof of Lemma \ref{lem:helper_rsmd}]
First, note that the ancillary sequence $(h_t)$ defined by (\ref{eqn:rsmd_update_proximal}) is indeed well-defined, since $\HH$ is closed and bounded, and $\VV$ is a reflexive Banach space.\footnote{\citet[Thm.~2.11]{barbu2012ConvOptBanach}.} Let us now consider the closely related sequence $(h_t^{\prime})$ generated by optimization over the whole space $\VV$:
\begin{align}
\nonumber
h_t^{\prime} & \defeq \argmin_{h \in \VV} \left[ \langle \overbar{G}_t, h \rangle + \frac{1}{\beta_t} \breg(h;h_t) \right]\\
\label{eqn:helper_rsmd_1}
& = \argmin_{h \in \VV} \left[ \beta_t \langle \overbar{G}_t, h \rangle + \breg(h;h_t) \right].
\end{align}
We want to show that such a sequence $(h_t^{\prime})$ is also well-defined. To see this, note that trivial first-order optimality conditions for sub-differentiable convex functions tell us that $h_t^{\prime}$ must satisfy
\begin{align*}
0 & \in \partial\left[ \beta_t \langle \overbar{G}_t, h_t^{\prime} \rangle + \breg(h_t^{\prime};h_t) \right]\\
& = \left\{ \beta_t\overbar{G}_t + \nabla\Phi(h_t^{\prime}) - \nabla\Phi(h_t) \right\}.
\end{align*}
The equality above follows from basic convex sub-differential calculus rules and the differentiability of $\breg$.\footnote{\citet[Thm.~3.39]{penot2012CWOD}.} We can thus obtain the following chain of equivalences:
\begin{align*}
0 \in \partial\left[ \beta_t \langle \overbar{G}_t, h_t^{\prime} \rangle + \breg(h_t^{\prime};h_t) \right] & \iff \nabla\Phi(h_t) - \beta_t \overbar{G}_t = \nabla\Phi(h_t^{\prime})\\
& \iff h_t^{\prime} \in \partial\Phi^{\ast}\left( \nabla\Phi(h_t) - \beta_t \overbar{G}_t \right).
\end{align*}
Here $\Phi^{\ast}$ denotes the Fenchel conjugate of convex function $\Phi$. The second equivalence holds whenever $\VV$ is reflexive.\footnote{\citet[Thm.~2.33, Rmk.~2.35]{barbu2012ConvOptBanach}.} The $\strong$-strong convexity of $\Phi$ implies that $\Phi^{\ast}$ is differentiable.\footnote{See \citet[Lem.~15]{shalev2007PhD} for this fact. See \citet{kakade2009b} for more on the duality of strong convexity and smoothness.}  This in turn implies that the sub-differential of $\Phi^{\ast}$ at $\nabla\Phi(h_t) - \beta_t \overbar{G}_t$ contains only a single element.\footnote{\citet[Prop.~2.40]{barbu2012ConvOptBanach}.} As such, the sequence $(h_t^{\prime})$ satisfying (\ref{eqn:helper_rsmd_1}) exists and is well-defined. The dual relation between the two sequences $\nabla\Phi(h_t^{\prime}) = \nabla\Phi(h_t) - \beta_t \overbar{G}_t$ follows immediately from the previous paragraph. Furthermore, by applying Lemma \ref{lem:helper_md_general}, we obtain the desired primal projection link (\ref{eqn:rsmd_update_proj}) between the two sequences of interest.\footnote{In applying Lemma \ref{lem:helper_md_general}, first note the correspondences are $\HH \leftrightarrow W$, $h_t \leftrightarrow v$, $h_t^{\prime} \leftrightarrow v^{\ast}$, $h_{t+1} \leftrightarrow \widetilde{v}^{\ast}$, and $f(\cdot) \leftrightarrow \langle \overbar{G}_t, \cdot \rangle$. Also, note that we have already proved the minimizers are well-defined.}
\end{proof}

\bibliographystyle{../refs/apalike}
\bibliography{../refs/refs.bib}

\end{document}